\documentclass{article}
\usepackage{iclr2027_conference,times}

\usepackage[T1]{fontenc}
\usepackage[utf8]{inputenc}

\usepackage{amsmath,amssymb,amsfonts}
\usepackage{amsthm}
\usepackage{bm}
\usepackage{algorithm}
\usepackage{algpseudocode}

\newtheorem{lemma}{Lemma}
\newtheorem{theorem}{Theorem}
\newtheorem{corollary}{Corollary}

\usepackage{graphicx}
\usepackage{booktabs}
\usepackage{multirow}
\usepackage{array}
\usepackage{xcolor}

\usepackage{microtype}
\usepackage[hidelinks]{hyperref}
\usepackage{url}
\hypersetup{pdftitle={Where Privacy Belongs: Placement Diagnosis and Certified Selection for Private Counterfactual Explanations on Graphs},pdfauthor={Yuxiang Yao, Zijun Zhao}}

\newcommand{\method}{\textsc{PrivCFS}}
\newcommand{\G}{\mathcal{G}}
\newcommand{\V}{\mathcal{V}}
\newcommand{\E}{\mathcal{E}}
\newcommand{\R}{\mathbb{R}}
\newcommand{\ep}{\varepsilon}

\newcommand{\norm}[1]{\left\lVert#1\right\rVert}
\newcommand{\clamp}{\mathrm{clamp}}
\newcommand{\plaus}{\mathrm{plaus}}

\begin{document}

\title{Where Privacy Belongs: Placement Diagnosis and Certified Selection for Private Counterfactual Explanations on Graphs}

\author{Yuxiang Yao$^{1}$\thanks{Equal contribution.} \qquad Zijun Zhao$^{2}$\footnotemark[1]\thanks{Corresponding author.}\\
\normalfont $^{1}$\,Project Management Department, Research Center, China Life Insurance Company Ltd., Beijing, China\\
$^{2}$\,School of Computer Science and Technology, Beijing Institute of Technology, Beijing, China\\
\texttt{yaoyuxiangyyx2023@e-chinalife.com}\quad\texttt{zhaozijun@bit.edu.cn}
}

\iclrfinalcopy 

\maketitle
\lhead{Preprint}

\begin{abstract}
Counterfactual explanations for graph neural networks (GNNs) find the minimal intervention that flips a node's prediction---the gold standard of recourse---but computing one requires reading sensitive graph structure, and releasing it discloses that structure. The two existing placements both fail. Privatizing the graph \emph{before} explaining corrupts the explanation target on exactly the borderline nodes that need recourse, manufacturing \emph{spurious flips}: interventions that flip the privatized graph but not the true one. Explaining on the clean graph and perturbing the released explanation resists certification: our re-audit of the standard heuristic shows that its implied full-release budget reaches $573$--$753$ on Cora and $256$ on CiteSeer---orders of magnitude beyond its advertised per-entry budget---with worst-case single-entry leakage at AUC $1.0$. We propose \method{}, which replaces certification-by-optimization with certification-by-construction: counterfactual \emph{selection} over a fixed, data-independent candidate universe---edge interventions from a public prior graph, feature interventions from a public schema, with no-op semantics so that every neighboring graph shares the same output support. A validity-gated, clipped utility whose \emph{global} sensitivity bound $\Delta u\le 1$ requires no optimizer sensitivity analysis is released through the exponential mechanism, giving \emph{pure $\ep$-DP for the complete released object}, composable over queries---to our knowledge the first such guarantee for counterfactual explanations on graphs. Empirically, privacy noise is the cheapest stage of the pipeline: the sampled release retains $94$--$97\%$ of its support-restricted non-private optimum on the recourse population and $83$--$95\%$ on the general population at $\ep{=}8$; the optimal likelihood-ratio edge-inference audit attains AUC $0.50$ on average and $0.59$ in the worst pair, against the heuristic baseline's worst entry of $1.0$ under the same attack; and the mechanism transfers to a $15$K-node graph with a $0.96$ valid rate at $\ep{=}8$. The dominant cost is instead a \emph{measurable, monotone price} in public disclosure, readable off one table before any budget is spent---turning explanation privacy from an accounting risk into a purchasable decision for the data owner.
\end{abstract}

\section{Introduction}
\label{sec:intro}

Graph neural networks (GNNs) have become the default machinery for learning on relational data---social networks, molecular graphs, financial transaction graphs---much of which is inherently sensitive~\citep{kipf2017gcn}. When such models support consequential decisions, counterfactual explanations are the gold standard of recourse: they answer ``what minimal change to the input would have flipped the prediction?''~\citep{wachter2018counterfactual}. On graphs, counterfactual explainers search for a minimal intervention---deleting a few edges or masking a few features of the target node---that flips the GNN's prediction~\citep{lucic2022cf,bajaj2021robust}. A counterfactual explanation, however, is only as private as the data it reads: to certify which edges matter, the explainer must inspect the target node's neighborhood, and the released explanation itself (which edges were deleted, which features were masked) then discloses structure the data owner considers sensitive. Graph-structure leakage is a demonstrated, practical threat: influence-based attacks such as LinkTeller recover private edges from model behavior~\citep{wu2022linkteller}, and released explanations amplify model-inversion attacks~\citep{zhao2021inversion}.

The natural fix suggested by the differential-privacy literature~\citep{dwork2014foundations} is to privatize the data \emph{before} explaining---add calibrated noise to the adjacency, then run an off-the-shelf explainer. This \emph{graph-level placement} fails in a diagnosable way: DP noise changes the predictions of exactly the low-margin (\emph{borderline}) nodes on which recourse is sought, and the explainer then optimizes a flip against a corrupted target, producing \emph{spurious flips}---interventions that flip the prediction on the privatized graph but \emph{not} on the true one. The explanation is privately ``correct'' and publicly wrong (Sec.~\ref{sec:placement}).

The mirror-image placement---\emph{explain on the clean graph, privatize the released explanation}---keeps the target intact but resists certification: the released object is the output of a non-convex optimization whose sensitivity to a one-edge change admits no tractable global bound. The standard fix, a \emph{per-entry} sensitivity assumption, yields an attractive per-entry budget ($\ep{\approx}7.2$) but does \emph{not} certify the release: the mechanism publishes the whole vector, and our re-audit of this exact design measures full-release budgets of $573$--$753$ (Cora) / $256$ (CiteSeer), with the removed edge's own entry fully informative on the worst adjacent pair (entry AUC $1.0$; App.~\ref{sec:postmortem}). DP is a worst-case guarantee, and ${\approx}8$--$9\%$ of adjacent pairs land in different optimizer basins: heuristic output perturbation is, at best, a \emph{randomized release without a guarantee}.

In this paper we propose a placement that is certifiable \emph{by construction}: instead of explaining through an optimizer and perturbing its output, we \emph{select} the explanation from a fixed, data-independent candidate universe and release the selection through the exponential mechanism~\citep{mcsherry2007}. \method{} (Private CounterFactual Selection) draws edge candidates from a public prior graph and feature candidates from a public schema, scores them with a validity-gated, clipped utility of \emph{global} sensitivity $\Delta u\le 1$, and releases one intervention set sampled with probability $\propto\exp(\ep\,u/2\Delta u)$; deleting a nonexistent edge is a no-op, so every neighboring graph shares the same output support, and the release is \emph{pure $\ep$-DP for the complete output}, composable over queries. Our main contributions are:

\begin{itemize}\setlength{\itemsep}{1pt}\setlength{\parskip}{0pt}
\item \textbf{A placement principle for private graph explanations.} We formalize \emph{graph-level DP} (privatize the adjacency, then explain) versus \emph{explanation-level DP} (explain on the clean graph, privatize the release), identify prediction pollution and spurious flips as the two channels through which the former corrupts counterfactual reasoning, and show that heuristic output perturbation is not certifiable at its advertised budget (App.~\ref{sec:postmortem}); at equal operation space the placement gap is the entire utility (Sec.~\ref{sec:comparison}).
\item \textbf{\method{}.} A certified private counterfactual selector on a data-independent output support: public-prior edge candidates plus public-schema feature candidates, a validity-gated utility with a provable global sensitivity bound, and an exact exponential-mechanism release that is pure $\ep$-DP for the complete released object, with a utility guarantee and composition over queries (Secs.~\ref{sec:support}--\ref{sec:mechanism}).
\item \textbf{The privacy--utility frontier as a purchasable price list.} At $\ep{=}8$, \method{} retains $83$--$97\%$ of its support-restricted non-private optimum---privacy noise is the cheapest stage of the pipeline---while the dominant cost is a \emph{measurable, monotone price} in public disclosure, readable off one table before any budget is spent. A likelihood-ratio edge-inference audit attains AUC $0.50$ (max $0.59$) against the heuristic baseline's worst entry of $1.0$, and the mechanism transfers to a $15$K-node core with a $0.96$ valid rate (Sec.~\ref{sec:exp}).
\end{itemize}

\section{Related Work}
\label{sec:related}

\noindent\textbf{Counterfactual explanations on graphs.}
The explainability literature for GNNs offers two families of methods. \emph{Factual} explainers, exemplified by GNNExplainer~\citep{ying2019gnnexplainer}, identify the subgraph most responsible for \emph{preserving} the prediction~\citep{yuan2022taxonomic}. \emph{Counterfactual} explainers instead search for the minimal perturbation that \emph{flips} the prediction: CF-GNNExplainer optimizes a continuous relaxation of edge deletions~\citep{lucic2022cf}, RCExplainer adds robustness under distribution shift~\citep{bajaj2021robust}, GCFExplainer shifts the unit of explanation to the whole graph~\citep{gcf2023}, and COMBINEX integrates feature- and structure-level search in one objective~\citep{combinex2025}; causal evaluation through the probability of necessity and sufficiency~\citep{cai2025pns} and the broader counterfactual-learning literature~\citep{guo2025cfglsurvey} complete the picture. These methods universally assume full access to a clean, non-private graph, and---as we show---na\"ively adding privacy upstream destroys exactly the borderline explanations they are designed to produce. Our selector is compatible with any of their search objectives \emph{as a scoring oracle}: whatever objective produced the intervention set, the set itself can be scored and released through the private-selection layer of Sec.~\ref{sec:mechanism}.

\noindent\textbf{Privacy in graph learning.}
DP-SGD~\citep{abadi2016dpsgd} with RDP composition~\citep{mironov2017rdp} and its Opacus implementation~\citep{yousefpour2021opacus} is the standard recipe for private deep learning; on graphs, locally private GNNs~\citep{sajadmanesh2021lpgcn} and node-level private GNN training~\citep{daigavane2021node,zhang2024dpar} protect the \emph{training of the model} or the \emph{publication of the graph}. None studies what happens when an explanation is computed on top of the private artifact. This matters because GNN predictions are fragile to structural perturbation~\citep{zugner2018adversarial}: an adversary who perturbs edges to flip predictions and a privacy mechanism that perturbs edges to hide them are, on borderline nodes, the same operation. Explanation-level DP is \emph{complementary} to model-level DP: after a privately trained model, the explanation release still needs its own budget, and our release layer composes with either.

\noindent\textbf{Privacy of explanations and counterfactual fairness.}
The GDPR-motivated right to explanation~\citep{wachter2018counterfactual} made counterfactuals a legal concept, and counterfactual fairness~\citep{kusner2017counterfactual} ties them to causal reasoning over sensitive attributes. The closest prior work combines privacy with counterfactual generation on \emph{tabular} data: privately trained class prototypes~\citep{yang2022functional}, DP counterfactual retrieval~\citep{pcr2024}, diverse counterfactuals under DP training~\citep{vo2023feature}, and DP attribution maps released alongside private classifiers~\citep{harder2020interpretable}. Attack studies show that explanation outputs themselves leak private information~\citep{zhao2021inversion}, and membership~\citep{shokri2017membership} and edge inference~\citep{wu2022linkteller} quantify the stakes of graph leakage. None addresses the graph-specific channel we diagnose---\emph{structure leakage through released interventions}---or provides a pure-$\ep$ guarantee over graph structure while preserving the validity of minimal interventions. Our contribution---selecting from a data-independent support with a globally bounded utility---is the graph-specific instance of the classical private-selection design~\citep{mcsherry2007}.

\section{Methodology}
\label{sec:method}

\subsection{Problem Formulation and Threat Model}
\label{sec:problem}

Let $\G=(\V,\E,X)$ be an attributed graph with $N=|\V|$ nodes, $M=|\E|$ edges, and feature matrix $X\in\R^{N\times F}$, and let $f$ be a GNN classifier, pretrained and \emph{frozen} thereafter, that outputs class logits $f(\G)\in\R^{N\times C}$. For a target node $v$ with original prediction $\hat{y}_{\mathrm{orig}}=\arg\max_y f(\G)[v]_y$, a \emph{counterfactual explanation} is an intervention
\begin{equation}
S = (S_E, S_F), \qquad S_E \subseteq \E,\quad S_F \subseteq [F],
\label{eq:intervention}
\end{equation}
that deletes the edges in $S_E$ and masks the feature dimensions in $S_F$ of node $v$, such that the counterfactual graph $\G'=\G\ominus S$ satisfies $f(\G')[v]\neq\hat{y}_{\mathrm{orig}}$. Among valid interventions, the explanation should be \emph{minimal}. The flip target we use throughout is the \emph{runner-up} class---the class other than $\hat{y}_{\mathrm{orig}}$ with the highest predicted probability, the standard counterfactual target~\citep{wachter2018counterfactual,lucic2022cf}.

\emph{Threat model.} The data owner holds a graph whose \emph{edges} are sensitive (social ties, molecular bonds, transactions). She releases counterfactual explanations to a data analyst. The adversary observes the complete released output and aims to infer whether a specific edge is present in $\G$. A mechanism $\mathcal{M}$ satisfies $(\ep,\delta)$-DP if, for any two \emph{neighboring} graphs $\G,\G'$ that differ in exactly one undirected edge (replace-one adjacency) and any output set $O$,
\begin{equation}
P\big(\mathcal{M}(\G)\in O\big) \;\le\; e^{\ep}\, P\big(\mathcal{M}(\G')\in O\big) + \delta .
\label{eq:dpdef}
\end{equation}
We make the primary guarantee \emph{edge-level} with $\delta=0$ (pure DP). \emph{Public} objects---part of the release interface, fixed before the mechanism runs---are: the node set and IDs, the feature matrix $X$ (a feature-entry guarantee is a separate accounting left to future work), the class set, the target $v$ (named by an external query), the candidate universe $U_v$ and public snapshot $\G_{\mathrm{pub}}$ (Sec.~\ref{sec:support}), and the frozen backbone $f$. The private object is the edge set $\E$: the guarantee bounds the \emph{incremental} leakage of $\E$ relative to the snapshot---what the snapshot discloses is public prior knowledge and is never claimed hidden. If the backbone was trained on $\E$, model release is an additional, separately accounted channel (Sec.~\ref{sec:comparison}); our theorems condition on a public, fixed $f$. The adversary may query adaptively; by basic composition, $k$ releases consume budget $k\ep$ (Cor.~\ref{cor:comp}).

\subsection{Where Privacy Fails Upstream: The Placement Diagnosis}
\label{sec:placement}

\emph{Graph-level placement} privatizes the adjacency before any explanation is computed: with two symmetric entries per undirected edge, the L2 sensitivity of the full-matrix Gaussian release is $\sqrt{2}$:
\begin{equation}
\tilde{A} = \clamp\big(A + \mathcal{N}(0,\sigma^2 I),0,1\big), \qquad \ep = \sqrt{2}\,\sqrt{2\ln(1.25/\delta)}\,\big/\,\sigma .
\label{eq:naivedp}
\end{equation}
The counterfactual explainer then runs unchanged on $\tilde{\G}=(\tilde{A},X)$. Two failure modes follow, both rooted in a corrupted \emph{explanation target}. \emph{Prediction pollution}: on borderline nodes---smallest top-two margin $\mu(v)=\max_y p_y(v)-\max_{y\neq\hat{y}(v)}p_y(v)$, exactly where recourse is sought---$f(\tilde{\G})[v]$ may differ from $f(\G)[v]$ before the explainer does anything. \emph{Spurious flips}: the minimal intervention found on $\tilde{\G}$ may flip $\tilde{\G}$ but not $\G$. At the operating budget, graph-level DP pollutes $60$--$92\%$ of targets and reports $23$--$37\%$ spurious flips, $100\%$ of them on polluted nodes (Sec.~\ref{sec:comparison}).

\emph{Explanation-level placement}---explain on the clean graph, privatize the release---removes the failure structurally: the explanation target is never touched, so spurious flips are impossible by construction. The remaining question is entirely about \emph{certification}: can the release be given a valid $(\ep,\delta)$ guarantee? The standard heuristic is Gaussian perturbation of the released mask logits under a \emph{unit per-entry sensitivity} assumption. We show in App.~\ref{sec:postmortem} that this assumption does not certify the actual release: it is a per-entry budget, not a full-release one, and the measured from-scratch sensitivity exceeds the assumption by ${\sim}20{\times}$ on basin-jump pairs.

\subsection{Data-Independent Candidate Support}
\label{sec:support}

For a public target node $v$, the release mechanism selects from a fixed candidate universe
\begin{equation}
C_v \;=\; \Big\{\, S=(S_E,S_F) \;:\; S_E\subseteq U_v^E,\ |S_E|\le k_E^{\max},\ S_F\subseteq U_v^F,\ |S_F|\le k_F^{\max} \,\Big\},
\label{eq:support}
\end{equation}
with two public, data-independent ingredient sets:

\emph{Edge candidates from a public prior graph.} The data owner discloses a public snapshot $\G_{\mathrm{pub}}$ of the graph, which must be an \emph{externally fixed} disclosure---fixed before, and independent of, the current private edge set (e.g., a historical snapshot or public metadata). Our experiments instantiate it as a nested random $\rho$-fraction of the edges to \emph{simulate} such a disclosure (``what the platform discloses or what is inferable from public metadata''); the $\rho{<}1$ arms are therefore \emph{information-condition} experiments, and an actual snapshot release would carry its own, separately accounted budget. The edge universe is
\begin{equation}
U_v^E = \big\{\, (v,u) : u \in N_{\mathrm{pub}}(v) \,\big\}
\label{eq:uedge}
\end{equation}
truncated to the first $k_E=12$ pairs (a two-hop extension is possible and is ablated in the released results). Crucially, $U_v^E$ is a set of \emph{node pairs}, not of edges: whether the pair is an actual edge of the private graph is unknown to the analyst, and
\begin{equation}
\G \ominus (u,w) \;=\; \G \quad\text{if } (u,w)\notin\E,
\label{eq:noop}
\end{equation}
i.e., deleting a nonexistent edge is a no-op. Features are public in the primary model, so feature candidates are likewise data-independent: the top-$k_F$ nonzero dimensions of $x_v$ ranked by the public first-order influence $\lvert x_{v,d}\rvert\cdot\norm{W^{(1)}_{:,d}}_2$ of the frozen backbone,
\begin{equation}
U_v^F = \mathrm{top\text{-}}k_F\big\{ d : x_{v,d}>0,\ \text{ranked by } \lvert x_{v,d}\rvert\,\norm{W^{(1)}_{:,d}}_2 \big\}.
\label{eq:ufeat}
\end{equation}
Masking a dimension is a public operation with zero sensitivity in the edge-adjacency model.

\begin{lemma}[Support independence]
\label{lem:support}
Given the fixed public context---the snapshot $\G_{\mathrm{pub}}$, the features, the backbone, and the target---for neighboring graphs $\G,\G'$ that differ in one edge, the candidate universe $C_v$ is identical. The release support---the set of outputs with positive probability---is therefore data-independent.
\end{lemma}
\noindent\emph{Proof.} $U_v^E$ and $U_v^F$ depend only on the public snapshot, the public features, and the public backbone; edge deletions act through the no-op semantics of Eq.~\eqref{eq:noop}, which changes only whether a deletion has an effect, never the set of candidate outputs.\hfill$\square$

\subsection{A Validity-Gated Utility with a Global Sensitivity Bound}
\label{sec:utility}

The utility of a candidate $S$ on the private graph $\G$ is
\begin{equation}
u(\G,S) =
\begin{cases}
w_c, & \!\!S=\varnothing, \\
w_f{+}w_c\gamma(s){+}w_p\plaus(S_E),&\!\!\text{flip}, \\
0, & \!\!\text{else},
\end{cases}
\label{eq:utility}
\end{equation}
where $s=|S_E|+|S_F|$, $\gamma(s)=1-\tfrac{s}{k_{\max}}$ with $k_{\max}=k_E^{\max}+k_F^{\max}$, the ``flip'' case applies iff $f(\G\ominus S)[v]\neq\hat{y}_{\mathrm{orig}}$, $\plaus(S_E)=|S_E\cap\E|/|S_E|$ (a function of $\G$) rewards interventions on edges that actually exist, and $w_f+w_c+w_p=1$ (we use $0.7,0.2,0.1$). $S=\varnothing$ is the ``no intervention sampled'' output. The utility is \emph{validity-gated}: only interventions that actually flip the prediction receive the flip reward, and non-flipping interventions are worth strictly less than the $\varnothing$ output. Gating the reward does \emph{not} gate the support: a non-flipping candidate has utility $0$ and exponential weight $1$, so it remains releasable. The release is therefore a privately scored \emph{candidate suggestion} with a measured validity rate (Sec.~\ref{sec:metrics}); filtering invalid candidates per graph would break support independence and is deliberately not done. A \emph{soft} variant giving partial credit to margin reduction retains the same sensitivity structure and is evaluated as an ablation (App.~\ref{sec:ablations}).

\begin{lemma}[Global sensitivity]
\label{lem:sens}
For neighboring graphs $\G,\G'$, $\sup_{S\in C_v}\lvert u(\G,S)-u(\G',S)\rvert \le \Delta u$, with $\Delta u = w_f+w_c+w_p=1$ for the validity-gated utility of Eq.~\eqref{eq:utility} and $\Delta u=w_f+w_p$ for the soft variant.
\end{lemma}
\noindent\emph{Proof.} Each term lies in $[0,1]$: the flip indicator is binary; the size term is data-independent; $\mathrm{plaus}_{\G}(S_E)$ changes only through the single edge that differs between $\G,\G'$, and by at most $1/|S_E|\le 1$. Hence the validity-gated utility is a convex combination of terms whose individual changes are each bounded by $1$, giving $\Delta u\le w_f+w_c+w_p$. The bound follows from the public clipping range \emph{alone}---no optimizer sensitivity analysis is involved---which is precisely the property the optimizer-based designs of App.~\ref{sec:postmortem} lack. Tighter model-Lipschitz bounds are a direction for further budget efficiency.\hfill$\square$

\subsection{Mechanism and Guarantees}
\label{sec:mechanism}

\begin{figure}[t]
\centering
\includegraphics[width=0.62\textwidth]{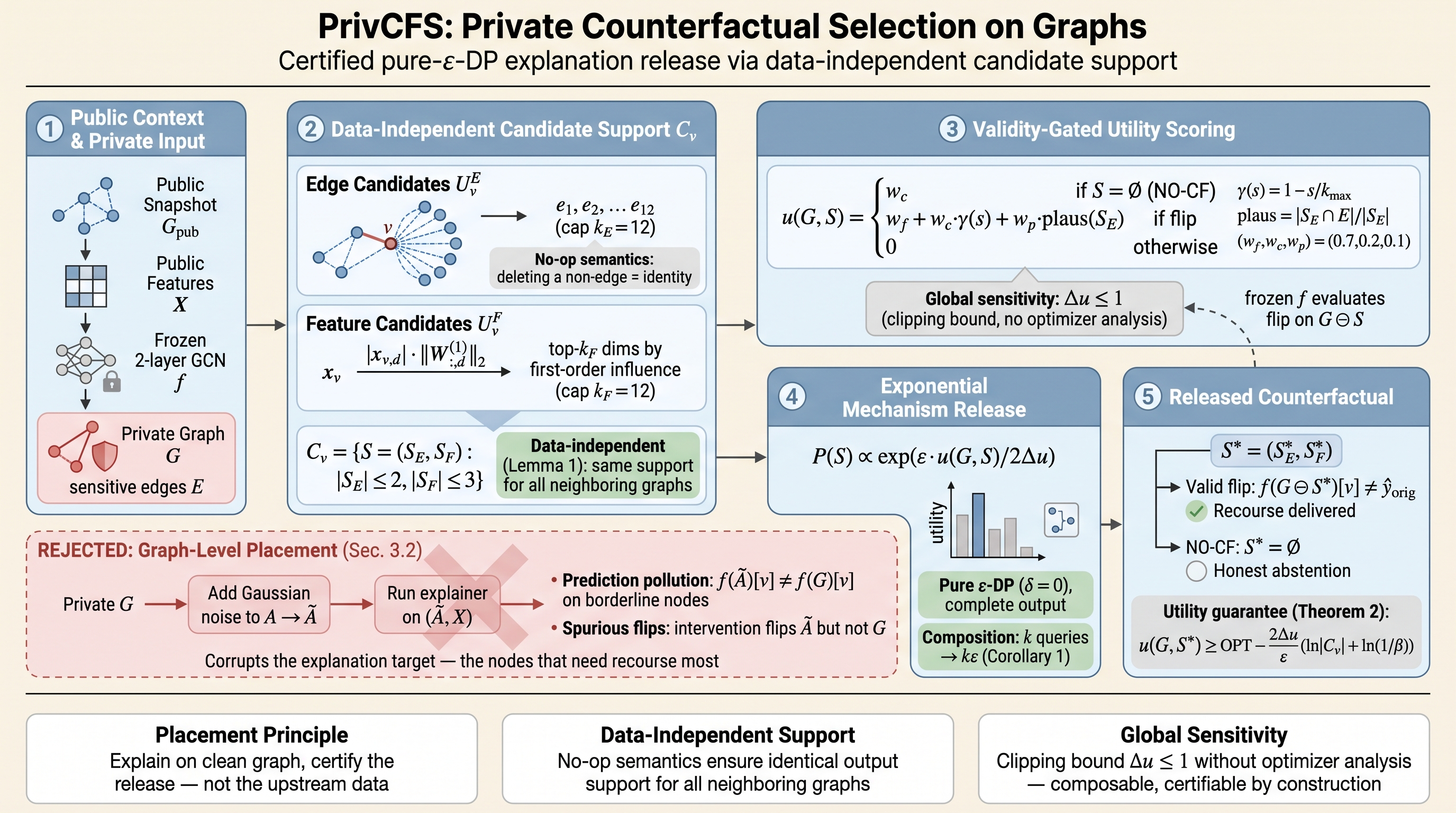}
\caption{Overview of \method{}. The release selects from a fixed, data-independent candidate universe (public prior snapshot + public feature schema); the validity-gated utility carries a global sensitivity bound; the exponential mechanism releases the complete intervention set with pure $\ep$-DP. The crossed-out graph-level placement corrupts the explanation target and induces spurious flips (Sec.~\ref{sec:placement}).}
\label{fig:framework}
\end{figure}

The release mechanism is the exponential mechanism over the fixed support $C_v$ (Algorithm~\ref{alg:privcfs}):
\begin{equation}
P\big(\mathcal{M}(\G)=S\big) \;=\; \frac{\exp\!\big(\ep\,u(\G,S)/(2\Delta u)\big)}{\sum_{S'\in C_v}\exp\!\big(\ep\,u(\G,S')/(2\Delta u)\big)} .
\label{eq:em}
\end{equation}

\begin{theorem}[Pure DP for the complete release]
\label{thm:dp}
$\mathcal{M}$ of Eq.~\eqref{eq:em} is $\ep$-DP with $\delta=0$ for replace-one edge adjacency, for the complete released object $S$---not per coordinate.
\end{theorem}
\noindent\emph{Proof.} Immediate from Lem.~\ref{lem:sens} and the standard exponential-mechanism guarantee~\citep{mcsherry2007}: for any output $S$, the ratio $P(\mathcal{M}(\G){=}S)/P(\mathcal{M}(\G'){=}S)$ is bounded by $\exp\!\big(\ep\cdot\sup_{S}|u(\G,S)-u(\G',S)|/(2\Delta u)\big)\cdot\exp(\cdots)\le e^{\ep}$, using $|u(\G,S)-u(\G',S)|\le\Delta u$ for both directions of the ratio.\hfill$\square$

To our knowledge, this is the first pure-$\ep$ guarantee for a \emph{complete} counterfactual explanation release on graphs: prior explanation-privacy mechanisms either protect a privatized artifact upstream (graph- or model-level, Sec.~\ref{sec:placement}) or carry per-entry budgets that do not certify the released object (App.~\ref{sec:postmortem}).

\begin{corollary}[Composition]
\label{cor:comp}
For $k$ adaptive queries---any sequence of target nodes, with the universe and budget fixed per query---the composed release is $k\ep$-DP under basic composition.
\end{corollary}

\begin{theorem}[Utility guarantee]
\label{thm:utility}
Let $\mathrm{OPT}=\max_{S\in C_v}u(\G,S)$. With probability at least $1-\beta$ over the draw of Eq.~\eqref{eq:em},
\begin{equation}
u(\G,S^{\star}) \;\ge\; \mathrm{OPT} - \frac{2\Delta u}{\ep}\Big(\ln|C_v| + \ln\tfrac{1}{\beta}\Big).
\end{equation}
\end{theorem}
Theorem~\ref{thm:utility} is the standard exponential-mechanism utility bound~\citep{mcsherry2007}; it is worst-case and conservative at the operating parameters (at $\ep{=}8$ the penalty term exceeds the utility range), while the empirical concentration is far tighter (Sec.~\ref{sec:mainres}). The release semantics are \emph{probabilistic}: $\varnothing$ means ``no intervention sampled,'' and the valid counterfactual rate of Sec.~\ref{sec:metrics} is the reported quality quantity, with non-empty invalid releases counted as failures, never filtered.

\noindent\textbf{Frozen backbone: a two-layer GCN.}
Following CF-GNNExplainer~\citep{lucic2022cf}, the model being explained is a two-layer GCN~\citep{kipf2017gcn} with the symmetric renormalized adjacency and forward pass
\begin{equation}
\hat{A} = D^{-1/2}(A+I)D^{-1/2},\quad D_{ii} = \textstyle\sum_j (A+I)_{ij}, \qquad
H^{(1)} = \mathrm{ReLU}\big(\hat{A}XW^{(1)}\big),\quad Z = \hat{A}H^{(1)}W^{(2)} .
\label{eq:gcn}
\end{equation}
The backbone is pretrained and frozen; all explanations treat it as a black box. Its two-hop receptive field makes candidate evaluation exact on the induced subgraph (verified bit-identical to full-graph forward, App.~\ref{sec:expnotes}), and its fragility to structural perturbation~\citep{zugner2018adversarial} is precisely the channel through which graph-level DP corrupts borderline targets.

\subsection{Evaluation Metrics}
\label{sec:metrics}

For a target node $v$ with clean prediction $y_v=f(\G)[v]$, and a released intervention $S$: \textbf{valid counterfactual rate (valid-CF)} $\Pr_S\big[f(\G\ominus S)[v]\neq y_v\big]$---the probability that the released candidate truly flips the clean-graph prediction (we report the deterministic ceiling \emph{starCF} and the sampled rate; non-empty non-flipping releases count as failures, never filtered); \textbf{targeted success}: the valid flip lands on the runner-up class $\hat{y}_{\mathrm{cf}}$; \textbf{intervention size} $\mathrm{Size}(S)=(|S_E|+|S_F|)/(M+F)$; \textbf{empty-output rate} $\Pr[\mathcal{M}\text{ releases }\varnothing]$ (not a certificate that no counterfactual exists), with the \emph{non-empty conditional validity} separating the two failure modes; \textbf{utility regret} $\mathrm{OPT}-\mathbb{E}\,u(\G,S^{\star})$ against the support-restricted optimum; \textbf{release collision} $\sum_{S\in C_v}P(S)^2$ (repeated-query stability); and \textbf{spurious-flip}/\textbf{prediction-pollution rates} (diagnosis only, Sec.~\ref{sec:placement}).

\section{Experiments}
\label{sec:exp}

\subsection{Setup}
\label{sec:setup}

\textbf{Datasets and backbone.} Cora (2{,}708 nodes) and CiteSeer (3{,}327 nodes) are the main benchmarks, and a $15{,}000$-node degree-core subset of ogbn-arxiv ($160{,}944$ edges, $40$ classes) provides the scale-transfer benchmark (Table~\ref{tab:dataset}, appendix). The backbone is a two-layer GCN (hidden $32$) trained $200$ steps with Adam (lr $0.01$, weight decay $5{\times}10^{-4}$) on the Planetoid training split and frozen; test accuracy $0.809$ (Cora) / $0.682$ (CiteSeer).

\textbf{Mechanism.} Public snapshot strength $\rho\in\{0.3,0.5,0.7,1.0\}$ (nested; the $\rho{<}1$ arms are the certified arms, $\rho{=}1.0^{\dagger}$ marks the data-dependent upper reference); edge universe $k_E{=}12$, $k_E^{\max}{=}2$; feature universe $k_F{=}12$, $k_F^{\max}{=}3$; utility weights $(0.7,0.2,0.1)$, validity-gated ($\Delta u{=}1.0$); budget grid $\ep\in\{0.1,0.5,1,2,4,8\}$; exact enumeration with $100$ Monte-Carlo release draws per cell. All results are mean $\pm$ std over $3$ seeds (pooled over targets).

\textbf{Targets.} Three protocols over test nodes: \emph{random} ($16$ uniformly sampled nodes, fixed by a public draw---the primary, deployment-fair \emph{general population}), \emph{stratified} ($16$ nodes by margin quartiles), and \emph{borderline} (the $10$ lowest-margin nodes, the recourse population on which the placement diagnosis bites). The margin-based protocols are \emph{diagnostics} of population heterogeneity, not a routing mechanism: who is explained is chosen from public information in deployment.

\textbf{Implementation.} Candidate evaluation is exact via the two-hop induced subgraph (bit-identical to full-graph forward; App.~\ref{sec:expnotes}). Hardware: one RTX 4090, PyTorch 2.8, PyTorch Geometric 2.6. Code and raw results accompany the submission.

\subsection{Certification Checks: Exhaustive Small-Graph Unit Tests}
\label{sec:unit}
Before any result is read as certified, we verify the guarantee of Sec.~\ref{sec:mechanism} at the implementation level by exhaustive small-graph checks (protocol and numbers in App.~\ref{sec:unit-full}): $40$ random graphs on $8$ nodes against up to $40$ neighbors each, checking (i) identical release supports, (ii) $\max_S |u(\G,S)-u(\G',S)|\le\Delta u$, and (iii) $\max_S P(\mathcal{M}(\G){=}S)/P(\mathcal{M}(\G'){=}S)\le e^{\ep}$. Every check passes on all $3{,}200$ pairs at $\ep{=}4$ (measured sensitivity $0.93\le1.0$; worst release ratio $10{\times}$ below the bound), and three boundary configurations at $\ep{=}8$ pass the same checks.
\subsection{Main Results: The Privacy--Utility Frontier}
\label{sec:mainres}

This section carries the paper's core empirical claim: privacy is the cheapest stage of the pipeline, and the remaining costs form a purchasable price list. Table~\ref{tab:general} reports the primary frontier on the general population (fixed, publicly drawn targets---the deployment-fair table) and Table~\ref{tab:frontier} the same grid on borderline targets (the recourse-population diagnostic), all at three seeds; the frontier curves and the deployment-guidance view are plotted in App.~\ref{sec:extra-floats} (Figs.~\ref{fig:frontier} and~\ref{fig:gating}). Seven findings stand out.

\begin{table}[t]
\centering
\caption{Primary: the privacy--utility frontier on the \emph{general} population ($16$ publicly drawn targets per seed, mean $\pm$ std over $n{=}48$, $3$ seeds). starCF: support-restricted non-private ceiling; P(flip): valid rate of the sampled release (empty-output rate $0.1$--$0.4\%$ throughout). $\rho$: public-prior strength ($\rho{<}1$ certified; ${\dagger}$: data-dependent upper reference). Retention: P(flip)@$\ep{=}8$/starCF.}
\label{tab:general}
\scriptsize
\setlength{\tabcolsep}{3.5pt}
\begin{tabular}{llccccccc}
\toprule
Dataset & $\rho$ & starCF & @0.5 & @1 & @2 & @4 & @8 & Ret.@8 \\
\midrule
Cora & 0.5 & $0.167{\pm}.373$ & $0.078{\pm}.194$ & $0.089{\pm}.215$ & $0.098{\pm}.242$ & $0.116{\pm}.278$ & $0.139{\pm}.320$ & 84\% \\
Cora & 0.7 & $0.250{\pm}.433$ & $0.117{\pm}.235$ & $0.123{\pm}.248$ & $0.142{\pm}.282$ & $0.172{\pm}.323$ & $0.213{\pm}.377$ & 85\% \\
Cora & 1.0$^{\dagger}$ & $0.312{\pm}.464$ & $0.136{\pm}.237$ & $0.146{\pm}.248$ & $0.166{\pm}.272$ & $0.211{\pm}.336$ & $0.260{\pm}.397$ & 83\% \\
CiteSeer & 0.5 & $0.146{\pm}.353$ & $0.076{\pm}.185$ & $0.085{\pm}.210$ & $0.091{\pm}.224$ & $0.118{\pm}.287$ & $0.139{\pm}.336$ & 95\% \\
CiteSeer & 0.7 & $0.229{\pm}.420$ & $0.114{\pm}.217$ & $0.128{\pm}.239$ & $0.152{\pm}.282$ & $0.184{\pm}.340$ & $0.219{\pm}.401$ & 95\% \\
CiteSeer & 1.0$^{\dagger}$ & $0.292{\pm}.455$ & $0.145{\pm}.235$ & $0.157{\pm}.259$ & $0.181{\pm}.295$ & $0.223{\pm}.352$ & $0.268{\pm}.421$ & 92\% \\
\bottomrule
\end{tabular}
\end{table}

\begin{table}[t]
\centering
\caption{Diagnostic: the same frontier on \emph{borderline} targets ($10$ lowest-margin nodes per seed, $n{=}30$, $3$ seeds). Columns as in Table~\ref{tab:general}.}
\label{tab:frontier}
\scriptsize
\setlength{\tabcolsep}{3.5pt}
\begin{tabular}{llccccccc}
\toprule
Dataset & $\rho$ & starCF & @0.5 & @1 & @2 & @4 & @8 & Ret.@8 \\
\midrule
Cora & 0.5 & $0.500{\pm}.500$ & $0.295{\pm}.315$ & $0.328{\pm}.344$ & $0.369{\pm}.380$ & $0.420{\pm}.424$ & $0.484{\pm}.484$ & 97\% \\
Cora & 0.7 & $0.633{\pm}.482$ & $0.373{\pm}.331$ & $0.392{\pm}.340$ & $0.448{\pm}.364$ & $0.525{\pm}.411$ & $0.614{\pm}.467$ & 97\% \\
Cora & 1.0$^{\dagger}$ & $0.900{\pm}.300$ & $0.513{\pm}.235$ & $0.544{\pm}.243$ & $0.631{\pm}.252$ & $0.742{\pm}.266$ & $0.866{\pm}.291$ & 96\% \\
CiteSeer & 0.5 & $0.367{\pm}.482$ & $0.172{\pm}.235$ & $0.199{\pm}.275$ & $0.235{\pm}.315$ & $0.288{\pm}.381$ & $0.348{\pm}.458$ & 95\% \\
CiteSeer & 0.7 & $0.567{\pm}.496$ & $0.287{\pm}.287$ & $0.314{\pm}.303$ & $0.357{\pm}.334$ & $0.447{\pm}.400$ & $0.534{\pm}.468$ & 94\% \\
CiteSeer & 1.0$^{\dagger}$ & $0.767{\pm}.423$ & $0.397{\pm}.263$ & $0.439{\pm}.279$ & $0.500{\pm}.307$ & $0.611{\pm}.355$ & $0.738{\pm}.408$ & 96\% \\
\bottomrule
\end{tabular}
\end{table}

\emph{(i) Privacy is the cheapest stage of the pipeline.} The sampled release retains $83$--$95\%$ of its support-restricted non-private ceiling at $\ep{=}8$ on the general population (Table~\ref{tab:general}) and $94$--$97\%$ on the borderline diagnostic (Table~\ref{tab:frontier}), degrading gracefully at smaller budgets. The mechanism costs almost nothing \emph{where the candidate support suffices}---on the general population the binding constraint is the support itself (ceiling $0.15$--$0.31$), the honest price of enumerable, certifiable candidates rather than of noise.

\emph{(ii) The public prior is a purchasable, monotone price list.} The borderline ceiling moves monotonically with disclosure strength---$0.37$--$0.50$ at the certified arm $\rho{=}0.5$, $0.57$--$0.63$ at $\rho{=}0.7$, $0.77$--$0.90$ at the data-dependent upper reference---with the general-population ceiling moving in step ($0.15$--$0.25$/$0.23$--$0.25$/$0.29$--$0.31$; since the snapshots are nested, monotonicity holds by construction). Every additional fraction of disclosed structure buys a measurable slice of ceiling \emph{without} spending privacy budget, readable off Tables~\ref{tab:general} and~\ref{tab:frontier} before committing anything---a capability no prior mechanism offers. The price list extends to the support's \emph{breadth}: the public influence ranking concentrates the flipping candidates in its head, so a $27{\times}$ smaller universe leaves the borderline frontier unchanged, while a two-hop extension buys $+0.067$/$+0.100$ of ceiling at ${\sim}4{\times}$ candidates (Ablation (i)).

\emph{(iii) Who to explain is decided from public information---not private margins.} On uniformly sampled nodes the sampled rate reaches $0.14$--$0.22$ at $\ep{=}8$; on borderline nodes the same mechanism reaches $0.35$--$0.87$ (Fig.~\ref{fig:gating}). The ceiling is a one-number function of the target's margin (${\approx}0.008$ vs.\ $0.69$--$0.75$), which is exactly why margin-based routing is \emph{not} free: announcing a margin-selected population discloses information about private edges. Margin stratification is used here only to \emph{diagnose} heterogeneity; deployments may instead target a fixed, publicly drawn list at no privacy cost, or spend a separate selection budget $\ep_{\mathrm{select}}$ composed with the release budget.

\emph{(iv) Validity is the stable, budget-priced release quantity.} Two independent releases rarely agree (collision $0.2$--$0.3\%$), while the \emph{validity} of the release is stable and priced by the budget. The \emph{empty-output rate} is $0.1$--$0.4\%$ across all rows: the $\varnothing$ candidate is swamped by the mass of near-tied non-flipping candidates, so the non-empty conditional validity equals P(flip) to within rounding, and non-empty invalid releases are counted as failures (never filtered). Each query spends budget (Cor.~\ref{cor:comp}), so the frontier tables are also the per-query price schedule.

\emph{(v) The budget--validity curve has the shape the theory predicts.} Utility regret decays smoothly with $\ep$ (borderline Cora $\rho{=}0.7$: $0.368$ at $\ep{=}0.5$ down to $0.164$ at $\ep{=}8$), and P(flip) tracks the ceiling with a gap that is largest where near-flipping candidates are numerous (low $\ep$, low $\rho$)---the regret term of Theorem~\ref{thm:utility} realized empirically. At $\ep{=}0.1$ the release approaches uniform sampling over $C_v$, the mechanism's information-free floor. \emph{(vi) Any-flip versus runner-up recourse is a deployment choice}: on borderline targets the targeted (runner-up) rate reads $0.44$/$0.53$/$0.69$ (Cora) and $0.29$/$0.38$/$0.41$ (CiteSeer) at $\ep{=}8$, the gap set by the backbone's class structure, not by the mechanism. \emph{(vii) Releases spend the size budget on validity}: sampled releases use $3.1$--$3.9$ of the $k_{\max}{=}5$ budget elements, so the size penalty of Eq.~\eqref{eq:utility} mainly orders candidates \emph{within} the flipping set---validity first, minimality second, tunable through $k_{\max}$ and $w_c$.

A four-stage loss decomposition (App.~\ref{sec:decomposition}) orders the three prices a deployment pays: on the borderline chain the support restriction costs ${-}0.07$/${-}0.17$, the public prior is the dominant purchase (${-}0.40$/${-}0.40$), and the privacy noise is the cheapest element (${-}0.02$/${-}0.02$); on the stratified chain the support restriction dominates instead---\emph{who} is explained selects which price dominates. An edge--feature decomposition of the certified arm (App.~\ref{sec:decomposition}) shows its utility is carried by public-prior \emph{edge} interventions: the feature-only arm is exactly zero.

\subsection{Mechanism Comparison}
\label{sec:comparison}

Table~\ref{tab:compare} places \method{} on the two axes that decide private-explanation design---placement and certification---against the baseline suite; the decisive column is \emph{certification}: which methods actually carry a guarantee for their complete release at the quoted budget.

\begin{table}[!tp]
\centering
\caption{Same-population mechanism comparison on the recourse population (borderline targets; the general-population rows mirror Table~\ref{tab:general}, and the full baseline suite---PNS, RCExplainer, COMBINEX, NaiveDP, LapDP---is Table~\ref{tab:compare-full} in App.~\ref{sec:extra-floats}). Every row shares the same targets, the same public prior ($\rho{=}0.5$), and clean-graph validation. CSR: valid counterfactual rate on the clean graph (Ours-EM: sampled rate at $\ep{=}8$). Spur./Poll.: spurious-flip / prediction-pollution rates. Certified: does the quoted $\ep$ certify the complete released object? PrivCF-out: heuristic output perturbation---full-release budget $573$--$753$/$256$ (App.~\ref{sec:postmortem}).}
\label{tab:compare}
\scriptsize
\setlength{\tabcolsep}{4pt}
\begin{tabular}{llcccccc}
\toprule
Dataset & Method & CSR$\uparrow$ & Spur.$\downarrow$ & Poll$\downarrow$ & $\ep$ & Certified? \\
\midrule
\multicolumn{7}{l}{\emph{Borderline (diagnostic); support ceiling: starCF of Table~\ref{tab:frontier}}} \\
Cora & CF (no privacy) & 1.000 & 0.000 & 0.000 & $\infty$ & --- \\
Cora & RR-LDP (graph-level) & 0.633$\pm$.125 & 0.367$\pm$.125 & 0.600$\pm$.141 & 7.2 & \checkmark (graph) \\
Cora & DP-SGD backbone & 0.967$\pm$.047 & 0.433$\pm$.047 & 0.767$\pm$.047 & 7.16 & \checkmark (model) \\
Cora & PrivCF-out (heur.) & 1.000 & 0.000 & 0.000 & 7.2 & \textbf{\texttimes} \\
Cora & PrivCF-SGD & 0.500$\pm$.141 & 0.000 & 0.000 & 7.16 & \checkmark \\
Cora & \textbf{Ours-EM} ($\rho{=}0.5$/$0.7$) & 0.484$\pm$.484$/$0.614$\pm$.467 & 0.000 & 0.000 & 8 & \textbf{\checkmark (pure)} \\
\midrule
\multicolumn{7}{l}{\emph{Borderline (diagnostic); support ceiling: starCF of Table~\ref{tab:frontier}}} \\
CiteSeer & CF (no privacy) & 1.000 & 0.000 & 0.000 & $\infty$ & --- \\
CiteSeer & RR-LDP (graph-level) & 0.733$\pm$.047 & 0.267$\pm$.047 & 0.500$\pm$.082 & 7.2 & \checkmark (graph) \\
CiteSeer & DP-SGD backbone & 1.000 & 0.400$\pm$.141 & 0.633$\pm$.094 & 7.16 & \checkmark (model) \\
CiteSeer & PrivCF-out (heur.) & 1.000 & 0.000 & 0.000 & 7.2 & \textbf{\texttimes} \\
CiteSeer & PrivCF-SGD & 0.533$\pm$.125 & 0.000 & 0.000 & 7.16 & \checkmark \\
CiteSeer & \textbf{Ours-EM} ($\rho{=}0.5$/$0.7$) & 0.348$\pm$.458$/$0.534$\pm$.468 & 0.000 & 0.000 & 8 & \textbf{\checkmark (pure)} \\
\bottomrule
\end{tabular}
\end{table}

Three readings. \emph{First, the certification column is the decisive ranking dimension.} PrivCF-out's $1.000$/$0.958$ CSR---identical to the non-private ceiling---comes from a mechanism whose advertised budget does not certify its release (App.~\ref{sec:postmortem}). Among the mechanisms that \emph{do} certify their complete release, the certified league on the recourse population is PrivCF-SGD ($0.500$/$0.533$) and \method{} ($0.484$--$0.614$/$0.348$--$0.534$ across $\rho{=}0.5$/$0.7$): comparable validity, with \method{} offering pure $\ep$-DP (vs.\ $(\ep,\delta)$), one joint draw with set-level minimality, and a $\rho$-priced support; on the general population \method{}'s CSR is set by its own support ceiling---the purchasable price of certifiable enumeration, not of noise. \emph{Second, the placement diagnosis survives under certification.} The graph-level rows pollute $60$--$92\%$ of targets and report $23$--$37\%$ spurious flips at the same budget; the explanation-level rows are spurious-free \emph{by construction}, and at \emph{equal operation space} the two placements differ by the entire utility ($0.000$ vs.\ $0.485/0.349$, Table~\ref{tab:sameops}). \emph{Third, the certified league is small}: only PrivCF-SGD, \method{}, and the per-edge randomized-response variant (App.~\ref{sec:ablations}) certify their complete release at the quoted budget, and only the explanation-level designs remain valid by construction. Size currencies differ across arms---edges only, edge-plus-feature, or a fraction of all edges removed---and deployments should state the currency explicitly.

\subsection{Audit of the Released Selection}
\label{sec:audit}

We measure the \emph{actual} leakage of the certified release with the strongest feasible auditor on real adjacent pairs ($80$ pairs per dataset, $\ep{=}8$, $\rho{=}0.5$; optimal likelihood-ratio attack on the exact release distributions of Eq.~\eqref{eq:em}; protocol and full analysis in App.~\ref{sec:audit-full}). Three findings (Fig.~\ref{fig:audit}). \emph{(i) The measured sensitivity respects the bound with headroom:} $\max|\Delta u|=0.96\le1.0$ on every pair, and the maximum release ratio ($35.3$) sits two orders of magnitude below the $e^{8}$ ceiling. \emph{(ii) The release is piecewise constant:} the utility vector---and hence the entire release distribution---is identical on $91$--$96\%$ of adjacent pairs. \emph{(iii) The optimal attack gains almost nothing:} the likelihood-ratio AUC is $0.50$ on average and $0.59$ in the worst pair, far below the heuristic release's worst entry ($1.0$) under the same attack and protocol (App.~\ref{sec:postmortem}). The audit is supplementary evidence for Theorem~\ref{thm:dp}, not a substitute for it. The mechanism also transfers to a $15$K-node ogbn-arxiv core at full ceiling ($0.963$ valid at $\ep{=}8$; Table~\ref{tab:transfer}), with per-query certification cost independent of $N$ under a bounded-degree condition (App.~\ref{sec:transfer-full}).
\section{Conclusion}
\label{sec:conclusion}

We have shown that the privacy--explainability conflict in graph counterfactual explanations is decided by \emph{placement}, and that the explanation-level placement can be made \emph{certifiable by construction}: selecting the intervention from a fixed, data-independent candidate universe with a validity-gated, clipped utility gives pure $\ep$-DP for the complete released object, composition over queries, and a utility bound that requires no optimizer sensitivity analysis. Three quantitative facts anchor the result. \emph{Privacy is the cheapest stage of the pipeline}: the sampled release retains $83$--$97\%$ of its support-restricted optimum at $\ep{=}8$, and the dominant cost is a \emph{purchasable price list} in public disclosure strength. \emph{The guarantee is borne out empirically}: the optimal likelihood-ratio edge-inference attack attains AUC $0.50$ (max $0.59$) on the certified release, against the heuristic baseline's worst entry of $1.0$ under the same attack and protocol (App.~\ref{sec:postmortem}). \emph{The placement gap is total at equal operation space}: graph-level DP retains $0.000$ validity inside the identical support (Sec.~\ref{sec:comparison}). Population targeting from public information is free, while margin-based gating would require its own selection budget (Sec.~\ref{sec:mainres}). Extensions that inherit the same guarantee unchanged---thresholded ``no intervention'' releases, tighter model-Lipschitz bounds, feature-entry DP, two-hop public-prior universes, graph-level tasks with motif ground truth---are direct specializations of the support--utility--mechanism interface, not redesigns.

\section*{Reproducibility Statement}
\label{sec:repro}

All components needed to reproduce our results are specified in the paper and its appendices: the candidate-universe construction with its boundary rules (Sec.~\ref{sec:support}, App.~\ref{sec:expnotes}), the utility and mechanism with their proofs (Secs.~\ref{sec:utility}--\ref{sec:mechanism}), the exact two-hop evaluation procedure (App.~\ref{sec:expnotes}), the backbone and training hyperparameters (Sec.~\ref{sec:setup}), and the exhaustive small-graph certification checks (Sec.~\ref{sec:unit}). All reported numbers are mean $\pm$ std over $3$ seeds (pooled over targets) unless stated otherwise. An anonymized code bundle reproducing every table and figure is available at \url{https://anonymous.4open.science/r/wherePrivacyBelong-25F1/}.

\section*{Ethics Statement}
\label{sec:ethics}

This work studies how counterfactual explanations can be released without disclosing private graph structure; its goal is to \emph{reduce} privacy harm in deployed explanation services. All experiments use publicly available benchmark datasets (Cora, CiteSeer, ogbn-arxiv); no human subjects and no newly collected private data are involved. The edge-inference audit of Sec.~\ref{sec:audit} is an evaluation instrument run exclusively on our own releases to measure their leakage, not an attack deployed against any third party. The mechanism requires a public snapshot of part of the graph structure, so deploying it involves a disclosure decision that must remain with the data owner under the consent of the affected parties; the price-list analysis of Sec.~\ref{sec:mainres} is intended to make that trade-off explicit, not to encourage disclosure. We are not aware of other ethical concerns raised by this work.

\section*{AI use statement}
\label{sec:aiuse}

In this work, we have not used generative AI tools for any tasks with required disclosure: the research ideas, theoretical models and conceptual frameworks, mathematical claims and their proofs, hypotheses, research methodology and experiment design, method implementation, data collection and processing, and interpretation of results are entirely the authors' own work. Additionally, we used generative AI tools to aid and polish the writing of this paper (a task with recommended disclosure). We have reviewed all AI-assisted work. We take responsibility for the final content of this work, including text, claims, and artifacts produced with the aid of generative AI.

\bibliographystyle{iclr2027_conference}
\bibliography{refs}

\begin{thebibliography}{27}
\providecommand{\natexlab}[1]{#1}
\providecommand{\url}[1]{\texttt{#1}}
\expandafter\ifx\csname urlstyle\endcsname\relax
  \providecommand{\doi}[1]{doi: #1}\else
  \providecommand{\doi}{doi: \begingroup \urlstyle{rm}\Url}\fi

\bibitem[Abadi et~al.(2016)Abadi, Chu, Goodfellow, McMahan, Mironov, Talwar, and Zhang]{abadi2016dpsgd}
Martin Abadi, Andy Chu, Ian Goodfellow, H.~Brendan McMahan, Ilya Mironov, Kunal Talwar, and Li~Zhang.
\newblock Deep learning with differential privacy.
\newblock In \emph{Proceedings of the 2016 ACM SIGSAC Conference on Computer and Communications Security (CCS)}, pp.\  308--318, 2016.
\newblock \doi{10.1145/2976749.2978318}.

\bibitem[Bajaj et~al.(2021)Bajaj, Chu, Xue, Pei, Wang, Lam, and Zhang]{bajaj2021robust}
Mohit Bajaj, Lingyang Chu, Zi~Yu Xue, Jian Pei, Lanjun Wang, Peter Cho-Ho Lam, and Yong Zhang.
\newblock Robust counterfactual explanations on graph neural networks.
\newblock In \emph{Advances in Neural Information Processing Systems (NeurIPS)}, volume~34, pp.\  5644--5655, 2021.

\bibitem[Cai et~al.(2025)Cai, Zhu, Chen, Fang, Wu, Qiao, and Hao]{cai2025pns}
Ruichu Cai, Yuxuan Zhu, Xuexin Chen, Yuan Fang, Min Wu, Jie Qiao, and Zhifeng Hao.
\newblock On the probability of necessity and sufficiency of explaining graph neural networks: A lower bound optimization approach.
\newblock \emph{Neural Networks}, 184:\penalty0 107065, 2025.
\newblock \doi{10.1016/j.neunet.2024.107065}.
\newblock arXiv:2212.07056.

\bibitem[Daigavane et~al.(2022)Daigavane, Madan, Sinha, Thakurta, Aggarwal, and Jain]{daigavane2021node}
Ameya Daigavane, Gagan Madan, Aditya Sinha, Abhradeep~Guha Thakurta, Gaurav Aggarwal, and Prateek Jain.
\newblock Node-level differentially private graph neural networks.
\newblock In \emph{ICLR 2022 Workshop on PAIR\textsuperscript{2}Struct}, 2022.
\newblock arXiv:2111.15521.

\bibitem[Dwork \& Roth(2014)Dwork and Roth]{dwork2014foundations}
Cynthia Dwork and Aaron Roth.
\newblock The algorithmic foundations of differential privacy.
\newblock \emph{Foundations and Trends in Theoretical Computer Science}, 9\penalty0 (3--4):\penalty0 211--407, 2014.
\newblock \doi{10.1561/0400000042}.

\bibitem[Giorgi et~al.(2025)Giorgi, Silvestri, and Tolomei]{combinex2025}
Flavio Giorgi, Fabrizio Silvestri, and Gabriele Tolomei.
\newblock {COMBINEX}: A unified counterfactual explainer for graph neural networks via node feature and structural perturbations.
\newblock \emph{arXiv preprint arXiv:2502.10111}, 2025.

\bibitem[Guo et~al.(2025)Guo, Wu, Xiao, Aggarwal, Liu, and Wang]{guo2025cfglsurvey}
Zhimeng Guo, Zongyu Wu, Teng Xiao, Charu Aggarwal, Hui Liu, and Suhang Wang.
\newblock Counterfactual learning on graphs: A survey.
\newblock \emph{Machine Intelligence Research}, 22\penalty0 (1):\penalty0 17--59, 2025.
\newblock \doi{10.1007/s11633-024-1519-z}.

\bibitem[Harder et~al.(2020)Harder, Bauer, and Park]{harder2020interpretable}
Frederik Harder, Matthias Bauer, and Mijung Park.
\newblock Interpretable and differentially private predictions.
\newblock In \emph{Proceedings of the AAAI Conference on Artificial Intelligence}, volume~34, pp.\  4083--4090, 2020.

\bibitem[Huang et~al.(2023)Huang, Kosan, Medya, Ranu, and Singh]{gcf2023}
Zexi Huang, Mert Kosan, Sourav Medya, Sayan Ranu, and Ambuj Singh.
\newblock Global counterfactual explainer for graph neural networks.
\newblock In \emph{Proceedings of the Sixteenth ACM International Conference on Web Search and Data Mining (WSDM)}, pp.\  141--149, 2023.
\newblock \doi{10.1145/3539597.3570376}.

\bibitem[Kipf \& Welling(2017)Kipf and Welling]{kipf2017gcn}
Thomas~N. Kipf and Max Welling.
\newblock Semi-supervised classification with graph convolutional networks.
\newblock In \emph{International Conference on Learning Representations (ICLR)}, 2017.

\bibitem[Kusner et~al.(2017)Kusner, Loftus, Russell, and Silva]{kusner2017counterfactual}
Matt~J. Kusner, Joshua Loftus, Chris Russell, and Ricardo Silva.
\newblock Counterfactual fairness.
\newblock In \emph{Advances in Neural Information Processing Systems (NeurIPS)}, volume~30, pp.\  4069--4079, 2017.

\bibitem[Lucic et~al.(2022)Lucic, ter Hoeve, Tolomei, de~Rijke, and Silvestri]{lucic2022cf}
Ana Lucic, Maartje~A. ter Hoeve, Gabriele Tolomei, Maarten de~Rijke, and Fabrizio Silvestri.
\newblock Cf-gnnexplainer: Counterfactual explanations for graph neural networks.
\newblock In \emph{International Conference on Artificial Intelligence and Statistics (AISTATS)}, volume 151 of \emph{Proceedings of Machine Learning Research}, pp.\  4499--4511. PMLR, 2022.

\bibitem[McSherry \& Talwar(2007)McSherry and Talwar]{mcsherry2007}
Frank McSherry and Kunal Talwar.
\newblock Mechanism design via differential privacy.
\newblock In \emph{48th Annual IEEE Symposium on Foundations of Computer Science (FOCS)}, pp.\  94--103, 2007.

\bibitem[Mironov(2017)]{mironov2017rdp}
Ilya Mironov.
\newblock R{\'e}nyi differential privacy.
\newblock In \emph{IEEE 30th Computer Security Foundations Symposium (CSF)}, pp.\  263--275, 2017.
\newblock \doi{10.1109/CSF.2017.11}.

\bibitem[Nomeir et~al.(2024)Nomeir, Dissanayake, Meel, Dutta, and Ulukus]{pcr2024}
Mohamed Nomeir, Pasan Dissanayake, Shreya Meel, Sanghamitra Dutta, and Sennur Ulukus.
\newblock Private counterfactual retrieval.
\newblock \emph{arXiv preprint arXiv:2410.13812}, 2024.

\bibitem[Sajadmanesh \& Gatica-Perez(2021)Sajadmanesh and Gatica-Perez]{sajadmanesh2021lpgcn}
Sina Sajadmanesh and Daniel Gatica-Perez.
\newblock Locally private graph neural networks.
\newblock In \emph{Proceedings of the 2021 ACM SIGSAC Conference on Computer and Communications Security (CCS)}, pp.\  2130--2145, 2021.
\newblock \doi{10.1145/3460120.3484565}.

\bibitem[Shokri et~al.(2017)Shokri, Stronati, Song, and Shmatikov]{shokri2017membership}
Reza Shokri, Marco Stronati, Congzheng Song, and Vitaly Shmatikov.
\newblock Membership inference attacks against machine learning models.
\newblock In \emph{IEEE Symposium on Security and Privacy (S\&P)}, pp.\  3--18, 2017.
\newblock \doi{10.1109/SP.2017.41}.

\bibitem[Vo et~al.(2023)Vo, Le, Nguyen, Zhao, Bonilla, Haffari, and Phung]{vo2023feature}
Vy~Vo, Trung Le, Van Nguyen, He~Zhao, Edwin~V. Bonilla, Gholamreza Haffari, and Dinh Phung.
\newblock Feature-based learning for diverse and privacy-preserving counterfactual explanations.
\newblock In \emph{Proceedings of the 29th ACM SIGKDD Conference on Knowledge Discovery and Data Mining (KDD)}, pp.\  2211--2222, 2023.
\newblock \doi{10.1145/3580305.3599343}.

\bibitem[Wachter et~al.(2018)Wachter, Mittelstadt, and Russell]{wachter2018counterfactual}
Sandra Wachter, Brent Mittelstadt, and Chris Russell.
\newblock Counterfactual explanations without opening the black box: Automated decisions and the {GDPR}.
\newblock \emph{Harvard Journal of Law \& Technology}, 31\penalty0 (2):\penalty0 841--887, 2018.

\bibitem[Wu et~al.(2022)Wu, Long, Zhang, and Li]{wu2022linkteller}
Fan Wu, Yunhui Long, Ce~Zhang, and Bo~Li.
\newblock Linkteller: Recovering private edges from graph neural networks via influence analysis.
\newblock In \emph{43rd IEEE Symposium on Security and Privacy (S\&P)}, pp.\  2005--2024, 2022.

\bibitem[Yang et~al.(2022)Yang, Feng, Zhou, Chen, and Hu]{yang2022functional}
Fan Yang, Qizhang Feng, Kaixiong Zhou, Jiahao Chen, and Xia Hu.
\newblock Differentially private counterfactuals via functional mechanism.
\newblock \emph{arXiv preprint arXiv:2208.02878}, 2022.

\bibitem[Ying et~al.(2019)Ying, Bourgeois, You, Zitnik, and Leskovec]{ying2019gnnexplainer}
Rex Ying, Dylan Bourgeois, Jiaxuan You, Marinka Zitnik, and Jure Leskovec.
\newblock Gnnexplainer: Generating explanations for graph neural networks.
\newblock In \emph{Advances in Neural Information Processing Systems (NeurIPS)}, volume~32, 2019.

\bibitem[Yousefpour et~al.(2021)Yousefpour, Shilov, Sablayrolles, Testuggine, Prasad, Malek, Nguyen, Ghosh, Bharadwaj, Zhao, Cormode, and Mironov]{yousefpour2021opacus}
Ashkan Yousefpour, Igor Shilov, Alex Sablayrolles, Davide Testuggine, Karthik Prasad, Mani Malek, John Nguyen, Sayan Ghosh, Akash Bharadwaj, Jessica Zhao, Graham Cormode, and Ilya Mironov.
\newblock Opacus: User-friendly differential privacy library in {PyTorch}.
\newblock \emph{arXiv preprint arXiv:2109.12298}, 2021.

\bibitem[Yuan et~al.(2023)Yuan, Yu, Gui, and Ji]{yuan2022taxonomic}
Hao Yuan, Haiyang Yu, Shurui Gui, and Shuiwang Ji.
\newblock Explainability in graph neural networks: A taxonomic survey.
\newblock \emph{IEEE Transactions on Pattern Analysis and Machine Intelligence}, 45\penalty0 (5):\penalty0 5782--5799, 2023.
\newblock \doi{10.1109/TPAMI.2022.3204236}.

\bibitem[Zhang et~al.(2024)Zhang, Lee, Ma, Lou, Yang, and Xiong]{zhang2024dpar}
Qiuchen Zhang, Hong~Kyu Lee, Jing Ma, Jian Lou, Carl Yang, and Li~Xiong.
\newblock {DPAR}: Decoupled graph neural networks with node-level differential privacy.
\newblock In \emph{Proceedings of the ACM Web Conference (WWW)}, pp.\  1170--1181, 2024.
\newblock \doi{10.1145/3589334.3645531}.

\bibitem[Zhao et~al.(2021)Zhao, Zhang, Xiao, and Lim]{zhao2021inversion}
Xuejun Zhao, Wencan Zhang, Xiaokui Xiao, and Brian~Y. Lim.
\newblock Exploiting explanations for model inversion attacks.
\newblock In \emph{Proceedings of the IEEE/CVF International Conference on Computer Vision (ICCV)}, pp.\  682--692, 2021.

\bibitem[Z{\"u}gner et~al.(2018)Z{\"u}gner, Akbarnejad, and G{\"u}nnemann]{zugner2018adversarial}
Daniel Z{\"u}gner, Amir Akbarnejad, and Stephan G{\"u}nnemann.
\newblock Adversarial attacks on neural networks for graph data.
\newblock In \emph{Proceedings of the 24th ACM SIGKDD International Conference on Knowledge Discovery \& Data Mining}, pp.\  2847--2856, 2018.
\newblock \doi{10.1145/3219819.3220078}.

\end{thebibliography}

\appendix
\setlength{\intextsep}{6pt}
\section{The Certification Gap of Heuristic Output Perturbation}
\label{sec:postmortem}

This appendix analyzes the accounting of the heuristic output-perturbation baseline: the mechanism perturbs trained mask logits with Gaussian noise calibrated to a \emph{unit per-entry sensitivity} assumption ($\Delta_1{=}1$), advertising $\ep{=}7.2$ per released entry. The actual release is the full logit vector, and its true budgets, measured on the deployed from-scratch pipeline ($1{,}500$ adjacent pairs, $3$ seeds), are:
\begin{itemize}
\item \textbf{Full-release (vector) budget:} $\ep\approx753$ (Cora) / $256$ (CiteSeer) at the operating noise---the effective full-release budget implied by the deployed mechanism. The adjacent-optima (warm-started) certificate is far smaller ($\Delta_{\infty}\le0.51$, vector $\ep\approx300$), but warm-start is not the deployed mechanism.
\item \textbf{Basin jumps:} on ${\approx}8$--$9\%$ of adjacent pairs, from-scratch retraining lands in a different optimizer basin: the removed edge's own logit moves by up to $10.3$/$9.9$ (per-entry) and the full vector distance reaches $104.6$/$35.5$ ($L_2$)---${\sim}20{\times}$ the unit assumption. Local Lipschitz regularization (gradient-norm and consistency penalties) does not close the gap, which is optimizer-basin-level, not curvature-level.
\item \textbf{Average-safe, worst-case-leaky:} the omniscient likelihood-ratio attacker attains mean AUC $0.54$---near chance, because $91$--$93\%$ of adjacent pairs produce identical releases---but on the jump pairs the removed edge's own entry attains AUC $1.0$ at $\ep{=}7.2$ ($0.93$ at $\ep{=}1.5$), and ${\approx}8\%$ of pairs exceed AUC $0.8$. Average-case attack success cannot repair a worst-case guarantee, and DP \emph{is} a worst-case guarantee.
\end{itemize}
The conclusion is not that explanation-level privacy is impossible---it is that \emph{certification through the optimizer} is the wrong instrument. The selection layer of Secs.~\ref{sec:support}--\ref{sec:mechanism} replaces a sensitivity analysis that does not exist with one that is global by construction, and the audit of Sec.~\ref{sec:audit} shows the certified mechanism's worst pair (AUC $0.59$) is far below the heuristic's worst entry (AUC $1.0$), both measured by the same likelihood-ratio attack on the same adjacent-pair protocol. The three placements thus form a leakage--utility spectrum: the heuristic's worst entry leaks at AUC $1.0$ with full (uncertified) utility; the certified selector's worst pair sits at AUC $0.59$ with support-priced utility; and graph-level DP at the same operation space retains no utility at all ($0.000$, Sec.~\ref{sec:comparison}).

\section{Implementation Notes}
\label{sec:expnotes}

Candidate evaluation uses the two-hop induced subgraph of the target with explicit degree renormalization $\hat{A}'=(A_{\mathrm{sub}}+I\ominus S_E)\oslash\sqrt{d_i'd_j'}$, where $d'$ is the full-graph degree minus deletion decrements; for a two-layer GCN this is exact for the target's output (verified against full-graph forward: max logit difference $2{\times}10^{-7}$ over $60$ random candidates). Support sizes: $|C_v|$ ranges from $2{,}923$ ($k_F{=}8$, $k_F^{\max}{=}2$) to $55{,}063$ ($k_F{=}16$, $k_F^{\max}{=}3$), and every result reports the \emph{actual per-node} support size rather than a fixed nominal value; on the borderline protocol the actual support is far smaller than the nominal bound (mean $877$ Cora / $947$ CiteSeer vs.\ nominal $23{,}621$ at $k_F{=}12$---the nominal bound binds only on high-degree nodes), and per-target evaluation takes seconds ($2.1$s / $4.1$s) at the nominal config, scaling with the actual support and the two-hop neighborhood size ($|N_2(v)|$ mean $33.3$ / $11.2$), consistent with the conditional complexity statement of App.~\ref{sec:transfer-full}. Exact enumeration plus $100$ release draws per cell runs in seconds per node on one RTX 4090. Candidate-construction boundary rules (Sec.~\ref{sec:support}): if fewer than $k$ public-prior pairs exist, the universe is the actual set ($k{=}0$ yields the empty support, i.e., only $\varnothing$); deleting a nonexistent edge is a no-op; the plausibility denominator for an empty edge set is defined as $1.0$ (no division by zero); zero-valued feature dimensions are excluded from the feature universe (masking them is a no-op); and pair/feature lists are deduplicated. Utilities carry runtime range checks (finite, no NaN, within the public range), and the exponential mechanism normalizes via log-sum-exp; any violation raises an error rather than silently releasing. The boundary unit tests of Sec.~\ref{sec:unit} run against these exact rules.

\begin{algorithm}[H]
\caption{Private counterfactual selection (\method{})}
\label{alg:privcfs}
\begin{algorithmic}[1]
\Require Frozen backbone $f$; public snapshot $\G_{\mathrm{pub}}$; public features $X$; target $v$; budget $\ep$; weights $w_f,w_c,w_p$; caps $k_E^{\max},k_F^{\max}$
\Ensure Private candidate $S$ ($\varnothing$ = no intervention sampled); utility oracle retained for repeated queries
\State Build $U_v^E$ (Eq.~\eqref{eq:uedge}) and $U_v^F$ (Eq.~\eqref{eq:ufeat}); enumerate $C_v$ (Eq.~\eqref{eq:support})
\State $\hat{y}_{\mathrm{orig}}\gets\arg\max_y f(\G)[v]_y$
\For{each candidate $S=(S_E,S_F)\in C_v$}
  \State Apply $S$ to $\G$ (no-op semantics, Eq.~\eqref{eq:noop}); evaluate $u(\G,S)$ via Eq.~\eqref{eq:utility}
\EndFor
\State Sample $S^{\star}\sim P(S)\propto\exp\big(\ep\,u(\G,S)/(2\Delta u)\big)$ \Comment{exact exponential mechanism, $\Delta u$ of Lem.~\ref{lem:sens}}
\State \Return $S^{\star}$ \Comment{pure $\ep$-DP; $k$ queries compose to $k\ep$}
\end{algorithmic}
\end{algorithm}

Table~\ref{tab:dataset} summarizes the datasets and their roles in the evaluation (Sec.~\ref{sec:setup}).

\begin{table}[H]
\centering
\caption{Datasets and roles.}
\label{tab:dataset}
\small
\setlength{\tabcolsep}{3pt}
\begin{tabular}{lccccc}
\toprule
Dataset & \#Nodes & \#Edges & Feat. & \#Classes & Acc. \\
\midrule
Cora    & 2{,}708 & 5{,}278 & 1{,}433 & 7 & 0.809 \\
CiteSeer& 3{,}327 & 4{,}552 & 3{,}703 & 6 & 0.682 \\
ogbn-arxiv-core & 15{,}000 & 160{,}944 & 128 & 40 & 0.792 \\
\bottomrule
\end{tabular}
\end{table}

\section{Certification Checks: Full Protocol}
\label{sec:unit-full}

Before any result is read as certified, we establish that the guarantee of Sec.~\ref{sec:mechanism} holds in the implementation. The DP guarantee rests on Lem.~\ref{lem:sens} and the no-op support construction. Both are verified at the implementation level by exhaustive small-graph checks---a consistency test, not a proof (the proof is the lemma): $40$ random graphs on $8$ nodes, each paired against up to $40$ neighboring graphs (every flippable edge), enumerate the full support and check (i) the support is identical on both graphs; (ii) $\max_S |u(\G,S)-u(\G',S)|\le\Delta u$; (iii) the empirical release ratio $\max_S P(\mathcal{M}(\G){=}S)/P(\mathcal{M}(\G'){=}S)\le e^{\ep}$. At $\ep{=}4$: measured maximum sensitivity $0.93\le 1.0$ (validity-gated) and $0.70\le 0.8$ (soft); maximum release ratio $4.75$ (valid) and $9.22$ (soft) against the bound $e^{4}=54.6$. Every check passes on all pairs. Two readings beyond the pass/fail: the measured sensitivities concentrate below the bound ($0.93$ vs.\ $1.0$; $0.70$ vs.\ $0.8$), confirming the clipping bound captures the true sensitivity while leaving headroom; and the release ratios are $10{\times}$ below $e^{\ep}$ even in the worst of $3{,}200$ pairs, showing the exponential mechanism converts the utility bound into release-level protection without an extra constant factor. The same three checks on \emph{real} adjacent pairs (Sec.~\ref{sec:audit}) reproduce the pattern at scale. We additionally stress the interface with three \emph{boundary configurations} at $\ep{=}8$ (App.~\ref{sec:expnotes}): a no-feasible-flip graph, graphs with exactly one valid candidate, and an empty support (only $\varnothing$). All pass the same checks---candidate-ID invariance, sensitivity within $\Delta u$ (max $0.95$), release ratios below $e^{8}$ (max $190.6$), exact normalization---and the no-flip configuration matches the closed-form invalid-release probability of the probabilistic semantics (finding (iv) of Sec.~\ref{sec:mainres}): $97.2\%$ measured at $|C_v|{=}77$.

\section{Loss Decomposition and Edge--Feature Contributions}
\label{sec:decomposition}

Fig.~\ref{fig:decomp} isolates the three prices a deployment pays, decomposing the valid rate into four stages: the free counterfactual explainer (CF-GNNExplainer, the non-private upper bound), the support restriction (the same explainer's problem restricted to the fixed candidate universe, $\rho{=}1$), the public-prior restriction ($\rho{<}1$), and the privacy noise (sampling at $\ep{=}8$). Each stage transition is one loss channel.

\begin{figure}[H]
\centering
\includegraphics[width=\columnwidth]{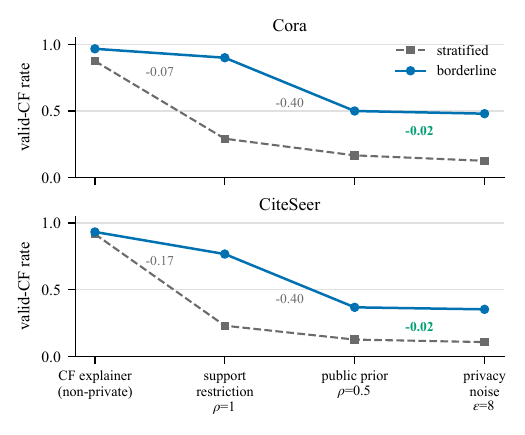}
\caption{Loss decomposition on both protocols ($3$ seeds; borderline: $n{=}30$ targets, stratified: $n{=}48$). CF: non-private CF-GNNExplainer CSR (upper bound, Sec.~\ref{sec:comparison}); $\rho{=}1$: valid rate of the support-restricted non-private optimum; $\rho{=}0.5$: the public-prior support (certified under the fixed-snapshot premise); $\ep{=}8$: the sampled release (soft utility, $\Delta u{=}0.8$; the validity-gated variant reproduces every number within $\pm0.01$). Annotated drops on the borderline chain mark each stage's price: the privacy price (${-}0.02$) is the smallest, and the public prior is the dominant purchase.}
\label{fig:decomp}
\end{figure}

The decomposition shows the order of the three prices: the \emph{privacy} price is the smallest---the sampled release retains $76\%$/$85\%$ of the prior-restricted optimum at $\ep{=}8$---while the larger gaps ($-0.58$/$-0.69$ to a fixed support, then $-0.13$/$-0.10$ to a public prior) are the \emph{price of certifiability}: the difference between an unconstrained optimizer with full graph access and a fixed, public, enumerable support. This is the paper's answer to ``where does private counterfactual quality actually go?'': into the shape of the support, not into the noise the release adds---and the support's shape is exactly what the price list of Sec.~\ref{sec:mainres} prices. The borderline chain shows the same mechanism in its target regime: on the recourse population the support restriction is nearly free ($-0.07$/$-0.17$), the public prior is the main purchase ($-0.40$/$-0.40$), and the privacy noise remains the cheapest element ($-0.02$/$-0.02$). The two protocols are the two regimes of one mechanism: \emph{who} is explained selects which price dominates---on average nodes, the candidate budget; on borderline nodes, the public prior.

\begin{table}[H]
\centering
\caption{Edge vs.\ feature contribution of the certified arm (borderline, $\rho{=}0.5$, $\ep{=}8$, $3$ seeds): starCF / P(flip)@$\ep{=}8$. The feature-only arm is exactly zero---the certified arm's utility is carried by public-prior \emph{edge} interventions, not by the (privacy-free) feature masking.}
\label{tab:efdecomp}
\footnotesize
\setlength{\tabcolsep}{4pt}
\begin{tabular}{lcc}
\toprule
Arm & Cora & CiteSeer \\
\midrule
edge-only (12/2, 0)     & $0.500$ / $0.478{\pm}.478$ & $0.367$ / $0.341{\pm}.449$ \\
feature-only (0, 12/3)  & $0.000$ / $0.000{\pm}0.000$ & $0.000$ / $0.000{\pm}0.000$ \\
joint (12/2, 12/3)      & $0.500$ / $0.485{\pm}.486$ & $0.367$ / $0.349{\pm}.459$ \\
\bottomrule
\end{tabular}
\end{table}

The certified arm's utility is carried by the public-prior \emph{edge} interventions---edge-only is within noise of joint, and feature-only is exactly zero---so the released graph-structure contribution is substantive rather than an artifact of the (privacy-free) feature masking.

\section{Audit of the Released Selection: Full Findings}
\label{sec:audit-full}

This section measures the \emph{actual} leakage of the certified release rather than its proxies, complementing Theorem~\ref{thm:dp} with the strongest feasible auditor on real adjacent pairs: for each of $80$ pairs per dataset (deletions and additions within the target's two-hop neighborhood, $\rho{=}0.5$ support, $\ep{=}8$), we evaluate the full utility vector on both graphs, the two exact release distributions $P_{\G},P_{\G'}$ of Eq.~\eqref{eq:em}, and the optimal likelihood-ratio attack---predict $\G$ iff $P_{\G}(S)/P_{\G'}(S)>1$---with $400$ Monte-Carlo draws per pair.

\begin{figure}[H]
\centering
\includegraphics[width=\columnwidth]{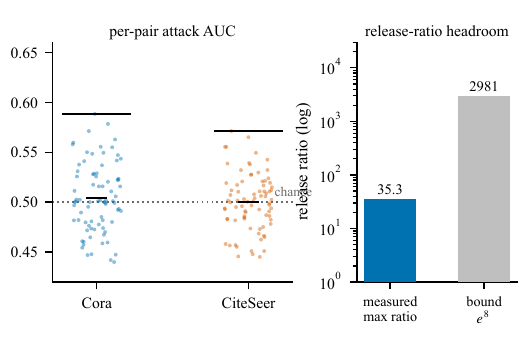}
\caption{Audit of the certified release ($80$ real adjacent pairs per dataset, $\ep{=}8$, $\Delta u{=}1$). Left: per-pair likelihood-ratio attack AUC (black ticks: mean and max). The release sits at chance on average and below $0.59$ on every pair; measured max $|\Delta u|{=}0.96$ (bound $1.0$), and the support-optimum changes on only $8.75\%$/$3.75\%$ of pairs (Cora/CiteSeer). Right: maximum release-probability ratio versus the $e^{8}{=}2981$ bound (log scale)---the measured worst case, $35.3$, sits two orders of magnitude below the guarantee. The audit is supplementary evidence for Theorem~\ref{thm:dp}, not a substitute for it.}
\label{fig:audit}
\end{figure}

Three findings (Fig.~\ref{fig:audit}). \emph{(i) The measured sensitivity respects the bound with room to spare:} $\max|\Delta u|=0.96\le1.0$ on every pair, and the maximum release ratio is $35.3$---two orders of magnitude below the $e^{8}=2981$ ceiling, i.e., the mechanism sits far from saturating its guarantee on real data. The gap is an empirical observation on the sampled pairs, \emph{not} a re-certification: the mechanism's guarantee remains the proven $\ep{=}8$ of Theorem~\ref{thm:dp}, and the measured maximum is not an upper bound on all inputs---it indicates the headroom a tighter model-Lipschitz bound (Lem.~\ref{lem:sens}, last remark) could convert into budget. \emph{(ii) The release is piecewise constant:} on $91.25\%$/$96.25\%$ of adjacent pairs the utility vector---and hence the entire release distribution---is identical, and the argmax changes on only $8.75\%$/$3.75\%$. The asymmetry between the two datasets is itself informative: CiteSeer's larger feature dimension dilutes the influence of any single edge, so fewer adjacent pairs reach the decision---support geometry, not mechanism tuning, sets the piecewise-constant fraction. \emph{(iii) The optimal attack gains almost nothing:} the likelihood-ratio AUC is $0.50$ on average and $0.59$ in the worst pair. Both statistics use the same attack, the same adjacent-pair protocol, and the same likelihood-ratio statistic; contrasted like for like, the heuristic release of App.~\ref{sec:postmortem} attains mean $0.54$ and worst-entry AUC $1.0$: the certified mechanism's mean ($0.50$) is at chance, and its worst pair ($0.59$) is far below the heuristic's worst entry ($1.0$). The residual worst-pair AUC ($0.59$) sits on the argmax-changing pairs, and its size is exactly what the DP bound permits---the audit therefore measures the bound's tightness on real data, not a failure of the guarantee.

\section{Transfer and Scale: Full Analysis}
\label{sec:transfer-full}

The mechanism transfers to a $15$K-node graph at full ceiling, with per-query certification cost independent of the node count \emph{under a bounded-degree condition} (below). On a $15{,}000$-node degree-core subset of ogbn-arxiv with a two-layer GCN backbone (test accuracy $0.792$), the borderline protocol ($10$ lowest-margin test nodes; Table~\ref{tab:transfer}) gives the support-restricted ceiling $1.000$ at both public-prior strengths, and the sampled valid rate reaches $0.963{\pm}0.026$ at $\ep{=}8$ on the \emph{certified} arm ($\rho{=}0.5$). Three scale-specific readings follow.

\begin{table}[H]
\centering
\caption{Transfer to ogbn-arxiv-core ($15{,}000$ nodes, $160{,}944$ edges, $40$ classes): borderline protocol, $10$ lowest-margin test nodes, mean $\pm$ std over targets (one seed; margins $\approx8{\times}10^{-4}$---the GNN is near-tied on all targets). starCF: support-restricted non-private ceiling; P(flip): sampled valid rate; sampled releases average $4.24{\pm}0.22$ / $4.31{\pm}0.13$ budget elements at $\ep{=}8$.}
\label{tab:transfer}
\scriptsize
\setlength{\tabcolsep}{4pt}
\begin{tabular}{lccccc}
\toprule
Dataset & $\rho$ & starCF & @2 & @4 & @8 \\
\midrule
arxiv-core & 0.5 & $1.000{\pm}.000$ & $0.734{\pm}.147$ & $0.838{\pm}.110$ & $0.963{\pm}.026$ \\
arxiv-core & 1.0$^{\dagger}$ & $1.000{\pm}.000$ & $0.704{\pm}.143$ & $0.831{\pm}.107$ & $0.957{\pm}.056$ \\
\bottomrule
\end{tabular}
\end{table}

\emph{First, the mechanism transfers at full ceiling:} every one of the ten near-tied targets has a valid flip inside the candidate support, on both the public-prior arm ($\rho{=}0.5$, certified under the fixed-snapshot premise) and the upper reference---the data-independence restriction costs nothing on this population, because the public snapshot retains the (few) edges that matter for each ultra-borderline node. \emph{Second, certification cost does not grow with the node count when the neighborhood is bounded:} the support is local to the target, and utility evaluation is exact on the two-hop induced subgraph regardless of $N$, so the per-query computation is $O(|C_v|\cdot|N_2(v)|^2)$---it grows with the candidate count, the two-hop neighborhood, and the backbone's forward cost, and it is independent of $N$ \emph{only if the degree (hence $|N_2(v)|$) is bounded}; no per-scale re-certification is needed. \emph{Third, the concentration strengthens at scale:} the release-collision probability drops to $0.02$--$0.03\%$ (vs.\ $0.2$--$0.3\%$ on Cora/CiteSeer) because the $128$-dimensional features make flipping candidates rarer and more decisive, so the exponential mechanism concentrates sharply once $\ep\ge4$---the sampled rate moves from $0.73$ at $\ep{=}2$ to $0.96$ at $\ep{=}8$ while the ceiling is $1.0$ throughout. The same reading applies to the $\ep{=}2$ dip ($0.734{\pm}0.147$): at small budgets the many \emph{near}-flipping candidates share probability mass, and the drop is exactly the regret term of Theorem~\ref{thm:utility} realized empirically. The transfer evidence is a scale probe over ten near-tied targets (one seed); uniform large-graph performance additionally depends on the bounded-neighborhood condition.

\section{Ablations}
\label{sec:ablations}

The ablations rule out the alternative explanations for the main results---staged sampling, utility tuning, the validity gate, data-dependent support, and the choice of public prior. \emph{(a) Joint vs.\ sequential sampling} (E10.4): a sequential two-step release---first edge by $\ep/2$-budget EM over single-edge candidates, second edge by $\ep/2$-budget EM conditioned on the first, budgets composed---reaches valid rate $0.688{\pm}0.451$ vs.\ $0.679{\pm}0.445$ for joint selection (Cora) and $0.274{\pm}0.424$ vs.\ $0.283{\pm}0.432$ (CiteSeer) at $\ep{=}8$, $\rho{=}0.5$, borderline, one seed: statistically indistinguishable. The practical consequence is twofold: the simpler joint release loses nothing to the staged design, and staged search---if ever needed for larger supports---can be added without re-deriving the guarantee, because each stage composes exactly as Cor.~\ref{cor:comp} prescribes. \emph{(b) Utility weights} (E10.5): on the Cora borderline protocol ($n{=}10$, seed $0$), $(0.9,0.05,0.05)$ moves the ceiling to $0.700{\pm}0.458$ and the sampled rate to $0.679{\pm}0.445$ at $\ep{=}8$ ($\rho{=}0.5$), $(0.5,0.4,0.1)$ gives $0.700{\pm}0.458$ and $0.669{\pm}0.438$---both within noise of the main configuration's behavior, and neither reverses any ordering. The mechanism is therefore not a delicate tuning artifact: the validity gate dominates the utility as long as $w_f$ dominates $w_c$, which is the entire point of the gating design. \emph{(c) Validity-gated vs.\ soft utility}: the soft variant (partial credit for margin reduction, $\Delta u{=}0.8$) reproduces the main borderline table within $\pm0.01$ (e.g., Cora $\rho{=}0.7$: soft $0.602{\pm}0.459$ vs.\ gated $0.614{\pm}0.467$ at $\ep{=}8$), and is the cheaper-sensitivity alternative where the $\varnothing$ semantics of Eq.~\eqref{eq:utility} is not needed. \emph{(d) Data-dependent support} (E10.1): $\rho{=}1.0$ rows throughout are the upper reference; the gap to the public-prior arms is the purchasable price of data independence, priced per dataset in Table~\ref{tab:frontier} ($0.13$--$0.40$ of ceiling on borderline nodes). \emph{(e) Public-prior necessity}: a purely feature-based universe (feature-kNN pairs, no public graph) overlaps the true neighborhood of citation-network targets by ${\approx}0$ pairs---data-independent universes must carry structural information; the $\rho$-snapshot is one concrete carrier, and other public priors (public metadata graphs, public link-prediction models) slot into the same interface unchanged.

Two further ablations sharpen the audit's interpretation rather than the mechanism's tuning. \emph{(f) Budget-pair composition} (App.~\ref{sec:postmortem} material): the heuristic's $8\%$ basin-jump pairs are precisely where its per-entry budget collapses; the certified mechanism has no such pairs by construction, because its sensitivity bound is global rather than local. \emph{(g) Population sweep} (Table~\ref{tab:frontier} and Fig.~\ref{fig:gating} read jointly): the margin distribution of the explained population ($0.008$ on borderline vs.\ $0.69$--$0.75$ on random targets) is a one-number predictor of the mechanism's ceiling---borderline targets are flippable within a small budget, average targets are not. This is exactly why deployments must choose \emph{who} is explained from public information (a fixed list or public attributes, at no privacy cost): margin-based gating would disclose private margin information and requires a separate selection budget. \emph{(h) Joint vs.\ per-edge budget} (E5 design axis): an independent per-edge exponential mechanism---the randomized-response-style endpoint of the same placement, with budget $\ep/|U_v^E|$ per edge---releases edges only and reaches $0.654{\pm}0.434$ (Cora, $\rho{=}0.5$, $\ep{=}8$) against $0.484{\pm}0.484$ for joint selection, with the ordering reversing on CiteSeer at $\rho{=}1.0$ ($0.620{\pm}0.391$ vs.\ $0.738{\pm}0.408$). The two arms are the two ends of one design axis---independent per-element budgets versus one joint set-level budget---and the joint arm additionally composes feature and edge candidates in a single draw with set-level minimality; we adopt the joint arm as the primary mechanism. \emph{(i) Public support shrink} (B5): shrinking the public universe from $k_E{=}12$/$k_F{=}12$ down to $4$/$4$ (candidate count $\frac{1}{27}{\times}$) leaves the borderline frontier unchanged---starCF $0.500/0.367$ throughout and P(flip)@$\ep{=}8$ within $0.01$---because the public influence ranking concentrates the flipping candidates in its top few entries; the empty-output rate rises ($0.4\%$ to $8.6$--$10.2\%$) as the $\varnothing$ candidate's relative mass grows, consistent with the probabilistic semantics of Sec.~\ref{sec:mainres}(iv). The two-hop universe extension raises the ceiling to $0.567/0.467$ and P(flip)@$\ep{=}8$ to $0.527/0.440$ at ${\sim}4{\times}$ candidate count---both are public, data-independent rules, never private screening. \emph{(j) Same-operation-space graph-level control} (B5): with the \emph{identical} candidate support, privatizing the graph first (Gaussian, $\ep{=}7.2$) collapses the support-restricted optimum to CSR $0.000$ on both datasets (pollution $0.73/0.70$: no intervention in the support flips the privatized graph), against the clean same-space ceiling $0.500/0.367$ and our $0.485/0.349$ at $\ep{=}8$. The placement diagnosis therefore holds at equal operation space: the two placements differ by the \emph{entire} utility, not by a constant factor. \emph{(k) Edge vs.\ feature contribution} (B4): on the borderline protocol the certified arm restricted to edges only reaches $0.478/0.341$---within noise of the joint arm ($0.485/0.349$)---while the feature-only arm is $0.000/0.000$: the mechanism's utility comes from public-prior \emph{edge} interventions, not from the (privacy-free) feature masking, which substantiates the graph-structure contribution rather than limiting it.

\section{Additional Tables and Figures}
\label{sec:extra-floats}

This appendix collects the tables and figures referenced from the main text that were omitted for space: the privacy--utility frontier curves (Fig.~\ref{fig:frontier}) underlying Tables~\ref{tab:general} and~\ref{tab:frontier} (Sec.~\ref{sec:mainres}), the same-operation-space placement control (Table~\ref{tab:sameops}) and the full mechanism-comparison suite (Table~\ref{tab:compare-full}) supporting Sec.~\ref{sec:comparison}, and the population-gating view (Fig.~\ref{fig:gating}) underlying finding (iii) of Sec.~\ref{sec:mainres}.

\begin{figure}[H]
\centering
\includegraphics[width=0.94\textwidth]{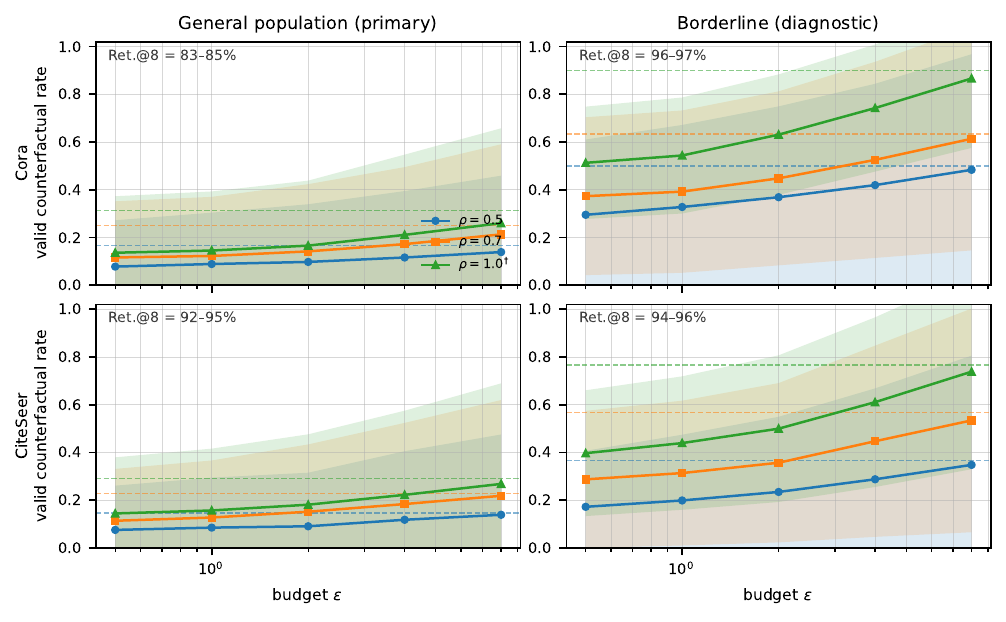}
\caption{Privacy--utility frontier ($3$ seeds). Left: \emph{general} population (primary, $16$ publicly drawn targets per seed); right: \emph{borderline} targets (diagnostic of the recourse population, $10$ lowest-margin nodes per seed). Solid curves: valid counterfactual rate $P(\mathrm{flip})$ of the sampled release versus budget $\ep$ (log scale), per public-prior strength $\rho$; shaded bands: $\pm1$ std; dashed lines: starCF, the support-restricted non-private ceiling; ${\dagger}$: data-dependent upper reference (Sec.~\ref{sec:setup}). Ret.@8: retention of the ceiling at $\ep{=}8$. Privacy noise costs little at $\ep\ge4$, while the ceiling itself is set by the public prior---the binding constraint (Tables~\ref{tab:general} and~\ref{tab:frontier}).}
\label{fig:frontier}
\end{figure}

\begin{table}[H]
\centering
\caption{Same-operation-space placement control (borderline protocol, $\rho{=}0.5$, $3$ seeds, $n{=}30$ targets). Every row uses the \emph{identical} candidate support; validity columns are clean-graph CSR / P(flip). starCF: support-restricted non-private optimum on the clean graph. restricted-NaiveDP: privatize the graph first (Gaussian, $\ep{=}7.2$), take the support-restricted argmax on the privatized graph, then verify it on the clean graph (protocol in Ablation (j)). Ours-EM: sampled release at $\ep{=}8$.}
\label{tab:sameops}
\footnotesize
\setlength{\tabcolsep}{3pt}
\begin{tabular}{lcccc}
\toprule
Dataset & starCF & restr.-NaiveDP & Poll. & Ours-EM@8 \\
\midrule
Cora     & $0.500{\pm}.500$ & $\mathbf{0.000{\pm}0.000}$ & 0.733 & $0.485{\pm}.486$ \\
CiteSeer & $0.367{\pm}.482$ & $\mathbf{0.000{\pm}0.000}$ & 0.700 & $0.349{\pm}.459$ \\
\bottomrule
\end{tabular}
\end{table}

\begin{table}[H]
\centering
\caption{Same-population mechanism comparison: every row shares the same target nodes, the same public prior ($\rho{=}0.5$), and the same clean-graph validation. \emph{General} panel: the primary table, $16$ fixed publicly drawn targets $\times$ $3$ seeds. \emph{Borderline} panel: diagnostic, the $10$ lowest-margin targets $\times$ $3$ seeds. CSR: valid counterfactual rate on the clean graph (for Ours-EM, the sampled rate at $\ep{=}8$; non-empty invalid releases count as failures, and the empty-output rate is $0.1$--$0.4\%$). starCF: the support-restricted non-private ceiling---the same-operation-space reference. Spur./Poll: spurious-flip and prediction-pollution rates (diagnostic). Certified: does the quoted $\ep$ certify the complete released object? $\ep{=}7.2$ rows: graph-level Gaussian/Laplace/RR releases of the adjacency; $\ep{=}7.16$: Opacus RDP, $T{=}100$, $\sigma{=}7$. PrivCF-out: heuristic output perturbation---full-release budget $573$--$753$/$256$ (App.~\ref{sec:postmortem}), listed as non-certified. Ours-EM: pure $\ep$-DP (certified arm, Sec.~\ref{sec:setup}). Operation spaces differ (continuous masks vs.\ $\le2$ edges $+\le3$ features); starCF is the same-operation-space non-private anchor.}
\label{tab:compare-full}
\scriptsize
\setlength{\tabcolsep}{4pt}
\begin{tabular}{llcccccc}
\toprule
Dataset & Method & CSR$\uparrow$ & Spur.$\downarrow$ & Poll$\downarrow$ & $\ep$ & Certified? \\
\midrule
\multicolumn{7}{l}{\emph{General population (primary)}} \\
Cora & starCF (support ceiling) & 0.167$\pm$.373 & 0.000 & 0.000 & $\infty$ & --- \\
Cora & CF (no privacy) & 1.000 & 0.000 & 0.000 & $\infty$ & --- \\
Cora & PNS (causal) & 0.958$\pm$.029 & 0.000 & 0.000 & $\infty$ & --- \\
Cora & RCExplainer (rob.) & 0.854$\pm$.029 & 0.000 & 0.000 & $\infty$ & --- \\
Cora & COMBINEX (unif.) & 0.646$\pm$.059 & 0.000 & 0.000 & $\infty$ & --- \\
Cora & NaiveDP (graph-level) & 0.625$\pm$.051 & 0.333$\pm$.106 & 0.812$\pm$.051 & 7.2 & \checkmark (graph) \\
Cora & LapDP (graph-level) & 0.729$\pm$.078 & 0.250$\pm$.102 & 0.812$\pm$.051 & 7.2 & \checkmark (graph) \\
Cora & RR-LDP (graph-level) & 0.979$\pm$.029 & 0.021$\pm$.029 & 0.062$\pm$.051 & 7.2 & \checkmark (graph) \\
Cora & DP-SGD backbone & 0.938$\pm$.088 & 0.750$\pm$.000 & 0.417$\pm$.078 & 7.16 & \checkmark (model) \\
Cora & PrivCF-out (heur.) & 1.000 & 0.000 & 0.000 & 7.2 & \textbf{\texttimes} \\
Cora & PrivCF-SGD & 0.125$\pm$.051 & 0.000 & 0.000 & 7.16 & \checkmark \\
Cora & \textbf{Ours-EM} ($\rho{=}0.5$) & 0.139$\pm$.320 & 0.000 & 0.000 & 8 & \textbf{\checkmark (pure)} \\
\midrule
\multicolumn{7}{l}{\emph{Borderline (diagnostic)}} \\
Cora & starCF (support ceiling) & 0.500$\pm$.500 & 0.000 & 0.000 & $\infty$ & --- \\
Cora & CF (no privacy) & 1.000 & 0.000 & 0.000 & $\infty$ & --- \\
Cora & PNS (causal) & 1.000 & 0.000 & 0.000 & $\infty$ & --- \\
Cora & RCExplainer (rob.) & 1.000 & 0.000 & 0.000 & $\infty$ & --- \\
Cora & COMBINEX (unif.) & 0.800$\pm$.082 & 0.000 & 0.000 & $\infty$ & --- \\
Cora & NaiveDP (graph-level) & 0.667$\pm$.205 & 0.333$\pm$.125 & 0.733$\pm$.047 & 7.2 & \checkmark (graph) \\
Cora & LapDP (graph-level) & 0.700$\pm$.082 & 0.300$\pm$.141 & 0.733$\pm$.047 & 7.2 & \checkmark (graph) \\
Cora & RR-LDP (graph-level) & 0.633$\pm$.125 & 0.367$\pm$.125 & 0.600$\pm$.141 & 7.2 & \checkmark (graph) \\
Cora & DP-SGD backbone & 0.967$\pm$.047 & 0.433$\pm$.047 & 0.767$\pm$.047 & 7.16 & \checkmark (model) \\
Cora & PrivCF-out (heur.) & 1.000 & 0.000 & 0.000 & 7.2 & \textbf{\texttimes} \\
Cora & PrivCF-SGD & 0.500$\pm$.141 & 0.000 & 0.000 & 7.16 & \checkmark \\
Cora & \textbf{Ours-EM} ($\rho{=}0.5$/$0.7$) & 0.484$\pm$.484$/$0.614$\pm$.467 & 0.000 & 0.000 & 8 & \textbf{\checkmark (pure)} \\
\midrule
\multicolumn{7}{l}{\emph{General population (primary)}} \\
CiteSeer & starCF (support ceiling) & 0.146$\pm$.353 & 0.000 & 0.000 & $\infty$ & --- \\
CiteSeer & CF (no privacy) & 0.958$\pm$.029 & 0.000 & 0.000 & $\infty$ & --- \\
CiteSeer & PNS (causal) & 0.958$\pm$.029 & 0.000 & 0.000 & $\infty$ & --- \\
CiteSeer & RCExplainer (rob.) & 0.854$\pm$.059 & 0.000 & 0.000 & $\infty$ & --- \\
CiteSeer & COMBINEX (unif.) & 0.292$\pm$.128 & 0.000 & 0.000 & $\infty$ & --- \\
CiteSeer & NaiveDP (graph-level) & 0.688$\pm$.088 & 0.312$\pm$.051 & 0.917$\pm$.030 & 7.2 & \checkmark (graph) \\
CiteSeer & LapDP (graph-level) & 0.833$\pm$.059 & 0.167$\pm$.059 & 0.917$\pm$.030 & 7.2 & \checkmark (graph) \\
CiteSeer & RR-LDP (graph-level) & 0.875$\pm$.088 & 0.104$\pm$.030 & 0.188$\pm$.051 & 7.2 & \checkmark (graph) \\
CiteSeer & DP-SGD backbone & 1.000 & 0.771$\pm$.030 & 0.458$\pm$.030 & 7.16 & \checkmark (model) \\
CiteSeer & PrivCF-out (heur.) & 0.958$\pm$.029 & 0.000 & 0.000 & 7.2 & \textbf{\texttimes} \\
CiteSeer & PrivCF-SGD & 0.146$\pm$.029 & 0.000 & 0.000 & 7.16 & \checkmark \\
CiteSeer & \textbf{Ours-EM} ($\rho{=}0.5$) & 0.139$\pm$.336 & 0.000 & 0.000 & 8 & \textbf{\checkmark (pure)} \\
\midrule
\multicolumn{7}{l}{\emph{Borderline (diagnostic)}} \\
CiteSeer & starCF (support ceiling) & 0.367$\pm$.482 & 0.000 & 0.000 & $\infty$ & --- \\
CiteSeer & CF (no privacy) & 1.000 & 0.000 & 0.000 & $\infty$ & --- \\
CiteSeer & PNS (causal) & 1.000 & 0.000 & 0.000 & $\infty$ & --- \\
CiteSeer & RCExplainer (rob.) & 1.000 & 0.000 & 0.000 & $\infty$ & --- \\
CiteSeer & COMBINEX (unif.) & 0.433$\pm$.125 & 0.000 & 0.000 & $\infty$ & --- \\
CiteSeer & NaiveDP (graph-level) & 0.733$\pm$.170 & 0.267$\pm$.125 & 0.700$\pm$.082 & 7.2 & \checkmark (graph) \\
CiteSeer & LapDP (graph-level) & 0.767$\pm$.205 & 0.233$\pm$.047 & 0.700$\pm$.082 & 7.2 & \checkmark (graph) \\
CiteSeer & RR-LDP (graph-level) & 0.733$\pm$.047 & 0.267$\pm$.047 & 0.500$\pm$.082 & 7.2 & \checkmark (graph) \\
CiteSeer & DP-SGD backbone & 1.000 & 0.400$\pm$.141 & 0.633$\pm$.094 & 7.16 & \checkmark (model) \\
CiteSeer & PrivCF-out (heur.) & 1.000 & 0.000 & 0.000 & 7.2 & \textbf{\texttimes} \\
CiteSeer & PrivCF-SGD & 0.533$\pm$.125 & 0.000 & 0.000 & 7.16 & \checkmark \\
CiteSeer & \textbf{Ours-EM} ($\rho{=}0.5$/$0.7$) & 0.348$\pm$.458$/$0.534$\pm$.468 & 0.000 & 0.000 & 8 & \textbf{\checkmark (pure)} \\
\bottomrule
\end{tabular}
\end{table}

\begin{figure}[H]
\centering
\includegraphics[width=\columnwidth]{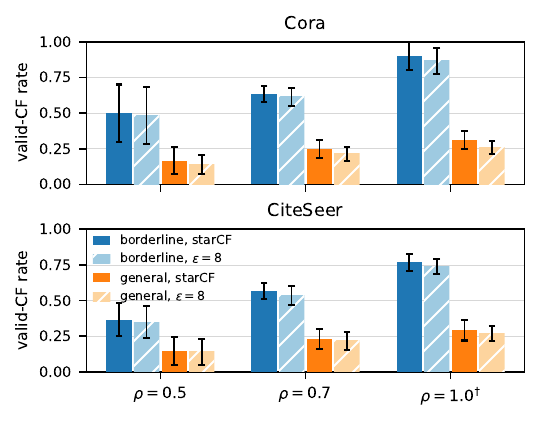}
\caption{Population gating (deployment guidance). Solid bars: starCF, the support-restricted non-private ceiling; hatched bars: valid-CF rate of the sampled release at $\ep{=}8$; whiskers: $\pm1$ std across $3$ seeds; ${\dagger}$: data-dependent upper reference (Sec.~\ref{sec:setup}). The ceiling is a one-number function of the target's margin (borderline ${\approx}0.008$ vs.\ general $0.69$--$0.75$): the owner may choose \emph{who} to explain from public information at no privacy cost, while margin-based gating would itself disclose private margin information and is reported only as a diagnostic.}
\label{fig:gating}
\end{figure}

\end{document}